\documentclass[sigconf]{acmart}
\AtBeginDocument{%
  }

\usepackage{tikz}
\usepackage{pgfplots}
\usepackage{xcolor}

\usepackage{booktabs}
\usepackage{caption}
\usepackage[labelformat=simple]{subcaption}

\usepackage{amsmath}
\usepackage{graphicx}
\usepackage{url}
\usepackage{hyperref}
\usepackage{amsfonts}
\usepackage{color, soul}
\usepackage{mathrsfs}
\usepackage{multirow}
\usepackage{float}
\usepackage{wrapfig}
\usepackage{amsthm}
\usepackage{bm}
\usepackage{bbm}

\usepackage{algorithm}  
\usepackage{algorithmic}   

\usepackage{colortbl}
\usepackage{enumitem}
\usepackage[T1]{fontenc}
\usepackage{aecompl}
\usepackage{dsfont}
\usepackage{mathrsfs}
\usepackage{bbm}

\usepackage{amsthm}
\usepackage{bm}
\usepackage{bbm}
\usepackage{algorithm}
\usepackage{algorithmic}
\usepackage{colortbl}
\usepackage{enumitem}
\usepackage[T1]{fontenc}
\usepackage{aecompl}
\usepackage{dsfont}
\usepackage{mathrsfs}
\usepackage{amsmath}

\usepackage{multicol} 

\usepackage[utf8]{inputenc} 
\usepackage[T1]{fontenc}    
\usepackage{booktabs}       
\usepackage{amsfonts}       
\usepackage{nicefrac}       
\usepackage{microtype}      
\usepackage{balance}
\usepackage{amsfonts}

\definecolor{mydarkblue}{rgb}{0,0.08,0.45}
\definecolor{myblue}{HTML}{3b75c3}
\definecolor{myred}{HTML}{E33222}
\definecolor{mygreen}{HTML}{438773}
\definecolor{mymaroon}{RGB}{142,27,19}
\definecolor{maroon}{HTML}{800000}
\definecolor{mycite}{cmyk}{0.55,1,0,0.15}
\definecolor{codeblue}{rgb}{0.25,0.5,0.5}
\definecolor{codekw}{rgb}{0.85, 0.18, 0.50}
\definecolor{codegreen}{rgb}{0,0.6,0}
\definecolor{codegray}{rgb}{0.5,0.5,0.5}
\definecolor{codepurple}{rgb}{0.58,0,0.82}
\definecolor{backcolour}{rgb}{0.95,0.95,0.92}
\hypersetup{
    colorlinks=true,
    citecolor=mymaroon,
    linkcolor=codeblue,
    urlcolor=maroon
          }

\usepackage[capitalize,noabbrev,nameinlink]{cleveref}

\creflabelformat{subfigure}{#2#1#3}
\newtheorem{definition}{Definition}
\newtheorem{problem}{Problem}
\newtheorem{proposition}{Proposition}
\newtheorem{theorem}{Theorem}

\copyrightyear{2026}
\acmYear{2026}
\setcopyright{cc}
\setcctype{by}
\acmConference[KDD '26]{Proceedings of the 32nd ACM SIGKDD Conference on Knowledge Discovery and Data Mining V.1}{August 09--13, 2026}{Jeju Island, Republic of Korea}
\acmBooktitle{Proceedings of the 32nd ACM SIGKDD Conference on Knowledge Discovery and Data Mining V.1 (KDD '26), August 09--13, 2026, Jeju Island, Republic of Korea}
\acmPrice{}
\acmDOI{10.1145/3770854.3780323}
\acmISBN{979-8-4007-2258-5/2026/08}

\begin{document}

\title{Certified Defense on the Fairness of Graph Neural Networks}

\author{Yushun Dong}
\affiliation{%
 \institution{Florida State University}
 \city{Tallahassee}
 \state{Florida}
 \country{USA}}
\email{yd24f@fsu.edu}

\author{Binchi Zhang}
\affiliation{%
 \institution{University of Virginia}
 \city{Charlottesville}
 \state{Virginia}
 \country{USA}}
\email{epb6gw@virginia.edu}

\author{Hanghang Tong}
\affiliation{%
 \institution{University of Illinois at Urbana-Champaign}
 \city{Urbana}
 \state{Illinois}
 \country{USA}}
\email{htong@illinois.edu}

\author{Jundong Li}
\affiliation{%
 \institution{University of Virginia}
 \city{Charlottesville}
 \state{Virginia}
 \country{USA}}
\email{jundong@virginia.edu}

\renewcommand{\shortauthors}{Yushun Dong, Binchi Zhang, Hanghang Tong, and Jundong Li}

\begin{abstract}
Graph Neural Networks (GNNs) have emerged as a prominent graph learning model in various graph-based tasks over the years. Nevertheless, due to the vulnerabilities of GNNs, it has been empirically shown that malicious attackers could easily corrupt the fairness level of their predictions by adding perturbations to the input graph data. In this paper, we take crucial steps to study a novel problem of certifiable defense on the fairness level of GNNs. Specifically, we propose a principled framework named ELEGANT and present a detailed theoretical certification analysis for the fairness of GNNs. ELEGANT takes {\em any} GNN as its backbone, and the fairness level of such a backbone is theoretically impossible to be corrupted under certain perturbation budgets for attackers. Notably, ELEGANT does not make any assumptions over the GNN structure or parameters, and does not require re-training the GNNs to realize certification. Hence it can serve as a plug-and-play framework for any optimized GNNs ready to be deployed. We verify the satisfactory effectiveness of ELEGANT in practice through extensive experiments on real-world datasets across different backbones of GNNs and parameter settings. 
\end{abstract}



\maketitle

\section{Introduction}

Graph Neural Networks (GNNs) have emerged among the most popular models to handle learning tasks on graphs~\cite{DBLP:conf/iclr/KipfW17,velivckovic2017graph} and made remarkable achievements in various domains~\cite{feng2022twibot,li2022graph}.
Nevertheless, as GNNs are increasingly deployed in real-world decision-making scenarios, there has been an increasing societal concern on the fairness of GNN predictions. A primary reason is that most traditional GNNs do not consider 
fairness, and thus could exhibit bias against certain demographic subgroups. Here the demographic subgroups are usually divided by certain sensitive attributes, such as gender and race.
To prevent GNNs from biased predictions, multiple recent studies have proposed fairness-aware GNNs~\cite{agarwal2021towards,dai2021say,kangwww2022,dong2023reliant} such that potential bias could be mitigated.

Unfortunately, despite existing efforts towards fair GNNs, it remains difficult to prevent the corruption of their fairness level due to their common vulnerability of lacking adversarial robustness.
In fact, malicious attackers can easily corrupt the fairness level of GNNs by perturbing the node attributes (i.e., changing the values of node attributes) and/or the graph structure (i.e., adding and deleting edges)~\cite{hussain2022adversarial}, which could lead to serious consequences in the test phase~\cite{dai2021say,hussain2022adversarial}.
For example, GNNs have been leveraged to perform bail decision-making on the graph of defendants, where an edge between two defendants represents high profile similarity~\cite{agarwal2021towards}.
Yet, by simply injecting adversarial links in the graph data, attackers can make GNNs deliver advantaged predictions for a subgroup (e.g., individuals with a certain nationality) while damaging the interest of others~\cite{hussain2022adversarial}. 
Hence achieving defense over the fairness of GNNs is crucial for the purpose of safe deployment.

It is worth noting that despite the abundant empirical defense strategies for GNNs~\cite{zhang2020gnnguard,entezari2020all,jin2019latent,jin2020graph,wu2019adversarial}, they are always subsequently defeated by novel attacking techniques~\cite{schuchardt2020collective,carlini2017adversarial}, and the defense over the fairness of GNNs also faces the same problem.
Therefore, an ideal way is to achieve certifiable defense on fairness (i.e., certified fairness defense).
A few recent works aim to certify the fairness for traditional deep learning models~\cite{khedr2022certifair,kang2022certifying,jin2022input,mangold2022differential,borca2022provable,ruoss2020learning}. 
Nevertheless, most of them require specially designed training strategies~\cite{khedr2022certifair,jin2022input,ruoss2020learning} and thus cannot be directly applied to optimized GNNs ready to be deployed.
More importantly, they mostly rely on assumptions on the optimization results~\cite{khedr2022certifair,jin2022input,borca2022provable,ruoss2020learning} or data distributions~\cite{kang2022certifying,mangold2022differential} over a continuous input space.
Hence they can hardly be generalized to GNNs due to the binary nature of the input graph topology.
Several other works propose certifiable GNN defense approaches to achieve theoretical guarantee~\cite{wang2021certified,bojchevski2019certifiable,bojchevski2020efficient,jin2020certified,zugner2019certifiable,zugner2020certifiable}.
However, they mainly focus on securing the GNN prediction for a certain individual node to ensure model utility, while how to achieve fairness defense over the entire population is ignored.
Therefore, despite the significance, the study in this field still remains in its infancy.

In fact, achieving certifiable defense on the fairness of GNNs is a daunting task due to the following key challenges:
(1) \textbf{Generality:}
different types of GNNs could be designed and optimized for different real-world applications~\cite{zhou2020graph}. Correspondingly, our first challenge is to design a plug-and-play framework that can achieve certified defense on fairness for any optimized GNN models that are ready to be deployed. 
(2) \textbf{Vulnerability:} a plethora of existing studies have empirically verified that most GNNs are sensitive to input data perturbations~\cite{zhang2020gnnguard,zugner2020adversarial,xu2019topology}. In other words, small input perturbations may cause significant changes in the GNN output. Hence our second challenge is to properly mitigate the common vulnerabilities of GNNs without changing its structure or re-training. 
(3) \textbf{Multi-Modality:} the input data of GNNs naturally bears multiple modalities. For example, there are node attributes and graph topology in the widely studied attributed networks. In practice, both data modalities may be perturbed by malicious attackers. Therefore, our third challenge is to achieve certified defenses of fairness on both data modalities for GNNs.

As an early attempt to address the aforementioned challenges, in this paper, we propose a principled framework named ELEGANT (c\textbf{\underline{E}}tifiab\textbf{\underline{LE}} \textbf{\underline{G}}NNs over the f\textbf{\underline{A}}ir\textbf{\underline{N}}ess of Predic\textbf{\underline{T}}ions).
Specifically, we focus on the widely studied task of node classification and formulate a novel research problem of \textit{Certifying GNN Classifiers on Fairness}. 
To handle the first challenge, we propose to develop ELEGANT on top of an optimized GNN model without any assumptions over its structure or parameters.
Hence ELEGANT is able to serve as a plug-and-play framework for any optimized GNN model ready to be deployed.
To handle the second challenge, we propose to leverage randomized smoothing~\cite{wang2021certified,cohen2019certified} to defend against malicious attacks, where most GNNs can then be more robust over the prediction fairness level.
To handle the third challenge, we propose two different strategies working in a concurrent manner, such that certified defense against the attacks on both the node attributes (i.e., add and subtract attribute values) and graph topology (i.e., flip the existence of edges) can be realized.
Finally, we evaluate the effectiveness of ELEGANT on multiple real-world network datasets. In summary, our contributions are three-fold: 
\begin{itemize}[topsep=5pt]
    \item \textbf{Problem Formulation.} We first formulate and perform an initial investigation on the novel research problem of \textit{Certifying GNN Classifiers on Fairness}.
    
    \item \textbf{Algorithm Design.} We propose a framework ELEGANT to achieve certified fairness defense against attacks on both node attributes and graph structure without relying on assumptions about any specific GNNs.
    
    \item \textbf{Experimental Evaluation.} We perform comprehensive experiments on real-world datasets under different GNN backbones to verify the effectiveness of ELEGANT.
\end{itemize}

\section{Problem Definition}

\noindent \textbf{Preliminaries.} 
Let $\mathcal{G} = \{\mathcal{V}, \mathcal{E}\}$ be an undirected attributed network, where $\mathcal{V} = \{v_1, ..., v_n\}$ is the set of $n$ nodes; $\mathcal{E} \subseteq \mathcal{V} \times \mathcal{V}$ is the set of edges.
Let $\bm{A} \in \{0,1\}^{n \times n}$ and $\bm{X} \in \mathbb{R}^{n \times d}$ be the adjacency matrix and attribute matrix of $\mathcal{G}$, respectively. 
Assume each node in $\mathcal{G}$ represents an individual, and sensitive attribute $s$ divides the population into different demographic subgroups. We follow a widely studied setting~\cite{agarwal2021towards,dai2021say} to assume the sensitive attribute is binary, i.e., $s \in \{0, 1\}$. We use $s_i$ to denote the value of the sensitive attribute for node $v_i$.
In node classification tasks, we use $\mathcal{V}_{\text{trn}}$ and $\mathcal{V}_{\text{tst}}$ ($\mathcal{V}_{\text{trn}}, \mathcal{V}_{\text{tst}} \in \mathcal{V}$) to represent the training and test node set, respectively.
We denote the GNN node classifier as $f_{\bm{\theta}}$ parameterized by $\bm{\theta}$. $f_{\bm{\theta}}$ takes $\bm{A}$ and $\bm{X}$ as input, and outputs $\hat{\bm{Y}}$ as the predictions for the nodes in $\mathcal{G}$. Each row in $\hat{\bm{Y}}$ is a one-hot vector flagging the predicted class.
Furthermore, we use $f_{\bm{\theta}^*}$ to denote a GNN equipped with the optimal parameter $\bm{\theta}^*$.

\noindent \textbf{Threat Model.} We focus on the attacking scenario of model evasion, i.e., the attack happens in the test phase. In particular, we assume that the victim model under attack is an optimized GNN node classifier $f_{\bm{\theta}^*}$. 
We follow a widely adopted setting~\cite{bojchevski2019certifiable,zugner2019certifiable,ma2020towards,mu2021hard} to assume that a subset of nodes $\mathcal{V}_{\text{vul}} \in \mathcal{V}_{\text{tst}}$ are vulnerable to attacks. 
Specifically, attackers may perturb their links (i.e., flip the edge existence) to other nodes and/or their node attributes (i.e., change their attribute values).
We denote the perturbations on adjacency matrix as $\bm{A} \oplus \bm{\Delta}_{\bm{A}}$. Here $\oplus$ denotes the element-wise XOR operator; $\bm{\Delta}_{\bm{A}} \in \{0,1\}^{n \times n}$ is the matrix representing the perturbations made by the attacker, where 1 only appears in rows and columns associated with the vulnerable nodes while 0 appears elsewhere.
Correspondingly, in $\bm{\Delta}_{\bm{A}}$, 1 entries represent edges that attackers intend to flip, while 0 entries are associated with edges that are not attacked.
Similarly, we denote the perturbations on node attribute matrix as $\bm{X} + \bm{\Delta}_{\bm{X}}$, where $\bm{\Delta}_{\bm{X}} \in \mathbb{R}^{n \times d}$ is the matrix representing the perturbations made by the attacker.
Usually, if the total magnitude of perturbations is within certain budgets (i.e., $\|\bm{\Delta}_{\bm{A}}\|_0 \leq \epsilon_{\bm{A}}$ for $\bm{A}$ and $\|\bm{\Delta}_{\bm{X}}\|_F \leq \epsilon_{\bm{X}}$ for $\bm{X}$), the perturbations are regarded as unnoticeable.
The goal of an attacker is to add unnoticeable perturbations to nodes in $\mathcal{V}_{\text{vul}}$, such that the GNN predictions for nodes in $\mathcal{V}_{\text{tst}}$ based on the perturbed graph exhibit as much bias as possible.
In addition, we assume that the attacker has access to any information about the victim GNN (i.e., a white-box setting). We note that such assumption aligns with the worst case in practice, which makes it even more challenging to achieve defense.

To defend against the aforementioned attacks, we aim to establish a node classifier on top of an optimized GNN backbone, such that this classifier, theoretically, will not exhibit more bias than a given threshold no matter what unnoticeable perturbations (i.e., perturbations within budgets) are added.
We formally formulate
the problem of \textit{Certifying GNN Classifiers on Fairness} below.
\begin{problem}
\label{p1}
\textbf{Certifying GNN Classifiers on Fairness.} Given an attributed network $\mathcal{G}$, a test node set $\mathcal{V}_{\text{tst}}$, a vulnerable node set $\mathcal{V}_{\text{vul}} \in \mathcal{V}_{\text{tst}}$, a threshold $\eta$ for the exhibited bias, and an optimized GNN classifier $f_{\bm{\theta}^*}$, our goal is to achieve a classifier on top of $f_{\bm{\theta}^*}$ associated with budgets $\epsilon_{\bm{A}}$ and $\epsilon_{\bm{X}}$, such that this classifier will bear comparable utility with $f_{\bm{\theta}^*}$ but provably not exhibit more bias than $\eta$ on the nodes in $\mathcal{V}_{\text{tst}}$, no matter what unnoticeable node attributes and/or graph structure perturbations (i.e., perturbations within budgets) are made over the nodes in $\mathcal{V}_{\text{vul}}$.
\end{problem}

\section{Methodology}

Here we first introduce the modeling of attack and defense on the fairness of GNNs, then discuss how we achieve certified defense on node attributes. After that, we propose a strategy to achieve both types of certified defense (i.e., defense on node attributes and graph structure) at the same time.
Finally, we introduce strategies to achieve the designed certified fairness defense for GNNs in practice.

\subsection{Bias Indicator Function}

We first construct an indicator $g$ to mathematically model the attack and defense on the fairness of GNNs. Our rationale is to use $g$ to indicate whether the predictions of $f_{\bm{\theta}^*}$ exhibit a level of bias exceeding a given threshold.
We present the formal definition below.
\begin{definition}
\label{indicator_func}
(Bias Indicator Function) Given adjacency matrix $\bm{A}$ and node attribute matrix $\bm{X}$, a test node set $\mathcal{V}_{\text{tst}}$, a threshold $\eta$ for the exhibited bias, and an optimized GNN model $f_{\bm{\theta}^*}$, the bias indicator function is defined as $g(f_{\bm{\theta}^*}, \bm{A}, \bm{X}, \eta, \mathcal{V}_{\text{tst}}) = \mathbf{1}\left( \pi(f_{\bm{\theta}^*}(\bm{A}, \bm{X}), \mathcal{V}_{\text{tst}}) < \eta \right)$,
where $\mathbf{1}(\cdot)$ takes an event as input and outputs 1 if the event happens (otherwise 0); $\pi(\cdot, \cdot)$ denotes any bias metric for GNN predictions (taken as its first parameter) over a set of nodes (taken as its second parameter). Traditional metrics measuring the exhibited algorithmic bias include $\Delta_{\text{SP}}$~\cite{dai2021say,DBLP:conf/innovations/DworkHPRZ12} and $\Delta_{\text{EO}}$~\cite{dai2021say,DBLP:conf/nips/HardtPNS16}.
\end{definition}
Correspondingly, the goal of the attacker is to ensure that the indicator $g$ outputs 0 for an $\eta$ as large as possible, while the goal of certified defense is to ensure for a given threshold $\eta$, the indicator $g$ provably yields 1 as long as the attacks are within certain budgets.
Note that a reasonable $\eta$ should ensure that $g$ outputs 1 based on the clean graph data (i.e., graph data without any attacks).
Below we first discuss the certified fairness defense over node attributes to maintain the output of $g$ as 1.

\subsection{Certified Defense over Node Attributes}

\label{noise1}

We now introduce how we achieve certified defense over the node attributes for the fairness of the predictions yielded by $f_{\bm{\theta}^*}$.
Specifically, we propose to construct a smoothed bias indicator function $\tilde{g}_{\bm{X}}(f_{\bm{\theta}^*}, \bm{A}, \bm{X}, \mathcal{V}_{\text{vul}}, \eta)$ via adding Gaussian noise over the node attributes of vulnerable nodes in $\mathcal{V}_{\text{vul}}$.
For simplicity, we use $\tilde{g}_{\bm{X}}(\bm{A}, \bm{X})$ to represent the smoothed bias indicator function over node attributes by omitting $\mathcal{V}_{\text{vul}}$, $f_{\bm{\theta}^*}$ and $\eta$.
We formally define $\tilde{g}_{\bm{X}}$ below.
\begin{definition}
\label{x_smooth}
(Bias Indicator with Node Attribute Smoothing) We define the bias indicator with smoothed node attributes over the nodes in $\mathcal{V}_{\text{vul}}$ as $\tilde{g}_{\bm{X}}(\bm{A}, \bm{X}) = \mathrm{argmax}_{c \in \{0, 1\}}\mathrm{Pr}( g(f_{\bm{\theta}^*}, \bm{A}, \bm{X} + \gamma_{\bm{X}}(\bm{\omega}_{\bm{X}}, \mathcal{V}_{\text{vul}}), \eta, \mathcal{V}_{\text{tst}}) = c)$.
Here $\bm{\omega}_{\bm{X}}$ is a $(d \cdot |\mathcal{V}_{\text{vul}}|)$-dimensional vector, where each entry is a random variable following a Gaussian Distribution $\mathcal{N}(0, \sigma^2)$; $\gamma_{\bm{X}}(\cdot, \cdot)$ maps a vector (its first parameter) to an $(n \times d)$-dimensional matrix, where the vector values are assigned to rows with the indices indicated by a set of nodes (its second parameter) while other entries are zeros.
\end{definition}

In the discussion below, we denote $\bm{\Gamma}_{\bm{X}} = \gamma_{\bm{X}}(\bm{\omega}_{\bm{X}}, \mathcal{V}_{\text{vul}})$ and $ g(\bm{A}, \bm{X} + \bm{\Gamma}_{\bm{X}}) = g(f_{\bm{\theta}^*}, \bm{A}, \bm{X} + \gamma_{\bm{X}}(\bm{\omega}_{\bm{X}}, \mathcal{V}_{\text{vul}}), \eta, \mathcal{V}_{\text{tst}})$ for simplicity.
Based on Definition~\ref{x_smooth}, we are now able to derive the theoretical certification for the defense on fairness with the defined $\tilde{g}_{\bm{X}}$ in~\cref{x_smooth}. Below we present a certification for the defense of fairness directly building upon the certification given in~\cite{cohen2019certified}. 
\begin{theorem}
\label{thm:1}
(Certified Fairness Defense for Node Attributes~\cite{cohen2019certified}) 
Denote the probability for $g(\bm{A}, \bm{X} + \bm{\Gamma}_{\bm{X}})$ to return class $c$ ($c \in \{0, 1\}$) as $P(c)$. 
Then $\tilde{g}_{\bm{X}}(\bm{A}, \bm{X})$ will provably return $\mathrm{argmax}_{c \in \{0, 1\}} P(c)$ for any perturbations (over the attributes of vulnerable nodes) within an $l_2$ radius $\tilde{\epsilon_{\bm{X}}} = \frac{\sigma}{2} \left(\Phi^{-1}(\max_{c \in \{0, 1\}} P(c)) - \Phi^{-1}(\min_{c \in \{0, 1\}} P(c)) \right)$, where $\Phi^{-1}(\cdot)$ is the inverse of the standard Gaussian cumulative distribution function. 
\end{theorem}
Correspondingly, for an $\eta$ that enables $\max_{c \in \{0, 1\}} P(c) = 1$, it is then safe to say that no matter what perturbations $\bm{\Delta}_{\bm{X}}$ are made on vulnerable nodes, as long as $\|\bm{\Delta}_{\bm{X}}\|_F \leq \tilde{\epsilon}_{\bm{X}}$, the constructed $\tilde{g}_{\bm{X}}$ will provably not yield predictions for $\mathcal{V}_{\text{tst}}$ with a level of bias exceeding $\eta$.
Nevertheless, it is worth noting that, in GNNs, perturbations may also be made on the structure of the vulnerable nodes, i.e., adding and/or deleting edges between these vulnerable nodes and any nodes in the graph. Hence it is also necessary to achieve certified defense against such structural attacks. Here we propose to also smooth the constructed $\tilde{g}_{\bm{X}}$ over the graph structure (of the vulnerable nodes) for the purpose of certified fairness defense on the graph structure. However, the adjacency matrix describing the graph structure is naturally binary, and thus should be smoothed in a different way. We elaborate on the joint certification over node attributes and graph structure in the next subsection.

\subsection{Certified Defense over Node Attributes and Graph Structure}

\label{noise2}

We then introduce achieving certified fairness defense against attacks on both node attributes and graph structure.
We propose a strategy to leverage noise following Bernoulli distribution to smooth $\tilde{g}_{\bm{X}}$ over the rows and columns (due to symmetricity) associated with the vulnerable nodes in $\bm{A}$. 
In this way, we can smooth both the node attributes and graph structure for $g$ in a randomized manner, and we denote the constructed function as $\tilde{g}_{\bm{A},\bm{X}}$.
We present the formal definition below.
\begin{definition}
\label{a_smooth}  
(Bias Indicator with Attribute-Structure Smoothing) We define the bias indicator function with smoothed node attributes and graph structure over the nodes in $\mathcal{V}_{\text{vul}}$ as $\tilde{g}_{\bm{A},\bm{X}}(\bm{A}, \bm{X}) = \mathrm{argmax}_{c \in \{0, 1\}} \mathrm{Pr}( \tilde{g}_{\bm{X}}(\bm{A} \oplus \gamma_{\bm{A}}(\bm{\omega}_{\bm{A}}, \mathcal{V}_{\text{vul}}), \bm{X}) = c)$.
Here $\bm{\omega}_{\bm{A}}$ is an $(n \cdot |\mathcal{V}_{\text{vul}}|)$-dimensional random variable, where each dimension takes 0 and 1 with the probability of $\beta$ ($0.5 < \beta \leq 1$) and $1 - \beta$, respectively; function $\gamma_{\bm{A}}(\cdot, \cdot)$ maps a vector (its first parameter) to a symmetric $(n \times n)$-dimensional matrix, where the vector values are assigned to rows whose indices associated with the indices of a set of nodes (its second parameter) and then mirrored to the corresponding columns, while other values are left as zeros.
\end{definition}

\begin{figure}[t]
    \centering
\includegraphics[width=0.8\linewidth]{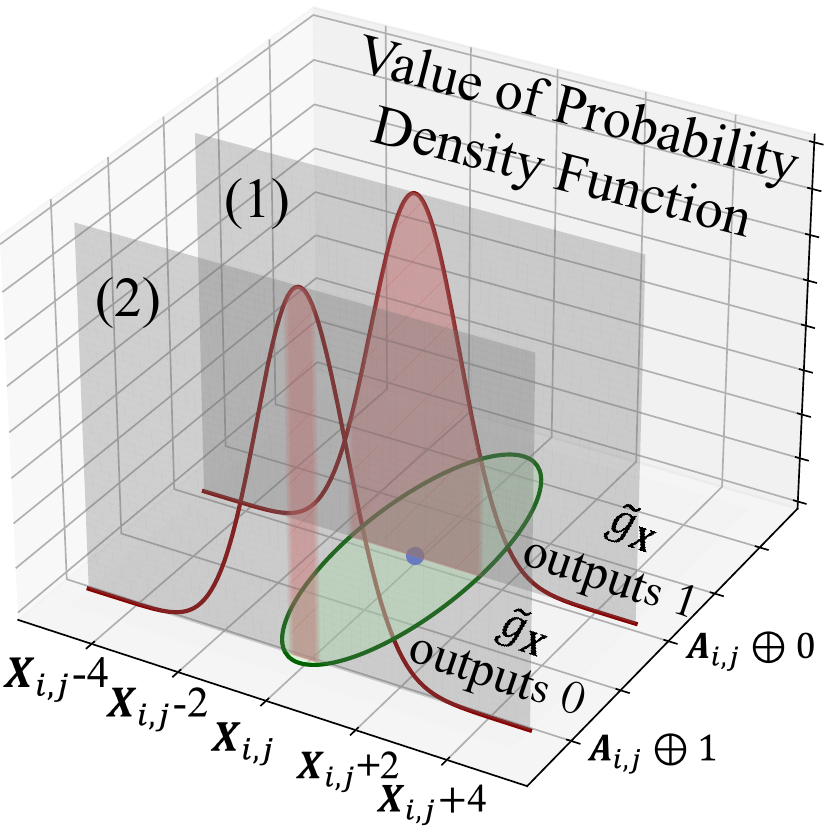}
    \caption{An example of how ELEGANT functions in the input space spanned by node attributes and graph structure.}  
    \label{dual_certification} 
\end{figure}

We let $\bm{\Gamma}_{\bm{A}} = \gamma_{\bm{A}}(\bm{\omega}_{\bm{A}}, \mathcal{V}_{\text{vul}})$ below for simplicity.
To better illustrate how classifier $\tilde{g}_{\bm{A},\bm{X}}$ achieves certified fairness defense over both data modalities of an attributed network, we provide an exemplary case in~\cref{dual_certification}.
Here we assume node $v_i \in \mathcal{V}_{\text{vul}}$. Considering the high dimensionality of node attributes and adjacency matrix, we only analyze two entries $\bm{X}_{i,j}$ and $\bm{A}_{i,j}$ and omit other entries after noise for simplicity. Here the superscript $(i,j)$ represents the $i$-th row and $j$-th column of a matrix.
Under binary noise, entry $\bm{A}_{i,j}$ only has two possible values, i.e., $\bm{A}_{i,j} \oplus 0$ and $\bm{A}_{i,j} \oplus 1$. We denote the two cases as Case (1) and Case (2), respectively.
We assume that the area where $g$ returns 1 in the span of the two input random entries of $g$ (i.e., $\bm{X}_{i,j}$ and $\bm{A}_{i,j}$ under random noise) is an ellipse (marked out with green), where the decision boundary is marked out with deep green.
In Case (1), $\bm{X}_{i,j}$ under random noise follows a Gaussian distribution, whose probability density function is marked out as deep red. 
We assume that, in this case, the integral of the probability density function within the range of the ellipse (marked out with shallow red) is larger than 0.5. Correspondingly, according to~\cref{x_smooth}, $\tilde{g}_{\bm{X}}$ returns 1 in this case.
In Case (2), we similarly mark out the probability density function and the area used for integral within the range of the ellipse. We assume that in this case, the integral is smaller than 0.5, and thus $\tilde{g}_{\bm{X}}$ returns 0.
Note that to compute the output of $\tilde{g}_{\bm{A},\bm{X}}$, we need to identify the output of $\tilde{g}_{\bm{X}}$ with the largest probability. Notice that $\beta > 0.5$, we have that $\bm{A}_{i,j} \oplus 0$ happens with a larger probability than $\bm{A}_{i,j} \oplus 1$.
Therefore, $\tilde{g}_{\bm{A},\bm{X}}$ outputs 1 in this example. In other words, the bias level of predictions of $f_{\bm{\theta}^*}$ is satisfying (i.e., smaller than $\eta$) according to the definition of $\tilde{g}_{\bm{A},\bm{X}}$.

Accordingly, to achieve certified fairness defense, our intuition here is to find tractable budgets $\epsilon_{\bm{A}}$ and $\epsilon_{\bm{X}}$, such that for any perturbations made over the node attributes and graph structure of the vulnerable nodes within $\epsilon_{\bm{A}}$ and $\epsilon_{\bm{X}}$, $\tilde{g}_{\bm{A},\bm{X}}$ provably maintains the same classification results in the way shown above. Below we introduce the certified fairness defense budgets over the graph structure $\epsilon_{\bm{A}}$ and node attributes $\epsilon_{\bm{X}}$ for $\tilde{g}_{\bm{A},\bm{X}}$.

\begin{theorem}
\label{a_r}
(Certified Defense Budgets for Structure \& Attribute Perturbations) 
The certified defense budget over the graph structure $\epsilon_{\bm{A}}$ for $\tilde{g}_{\bm{A},\bm{X}}$ is given as
\begin{align}
\label{solving_k}
    \epsilon_{\bm{A}} = \max \tilde{\epsilon_{\bm{A}}} , \; \text{s.t.} \; \underline{P_{\tilde{g}_{\bm{X}}=1}} > 0.5, \; \forall \;\|\bm{\Delta}_{\bm{A}}\|_{0} \leq \tilde{\epsilon_{\bm{A}}}
\end{align}
following the formulation given in~\cite{lee2019tight,bojchevski2020efficient,wang2021certified}. Here $\underline{P_{\tilde{g}_{\bm{X}}=1}}$ is the lower bound of $\mathrm{Pr}(\tilde{g}_{\bm{X}}(\bm{A} \oplus \bm{\Delta}_{\bm{A}} \oplus \bm{\Gamma}_{\bm{A}}, \bm{X})=1)$.

The certified defense budget $\epsilon_{\bm{X}}$ for $\tilde{g}_{\bm{A},\bm{X}}$ is given as
\begin{align}
\label{x_r}
    \epsilon_{\bm{X}} = \min \bigl \{\tilde{\epsilon_{\bm{X}}}: \tilde{\epsilon_{\bm{X}}} = \zeta \left( \tilde{g}_{\bm{X}}(\bm{A} \oplus \bm{\Gamma}_{\bm{A}}, \bm{X}) \right) \text{ , where} \; \bm{\Gamma}_{\bm{A}} \in \bar{\mathcal{A}} \bigl \}.
\end{align}
Here $\bar{\mathcal{A}}$ is the set of all possible $(n \times n)$-matrices, where entries in rows whose indices associate with those vulnerable nodes may take 1 or 0, while other entries are zeros. Given an $\bm{\Gamma}_{\bm{A}} \in \bar{\mathcal{A}}$, $\zeta(\cdot)$ outputs the perturbation budget over node attribute $\bm{X}$ as given in Theorem~\ref{thm:1}.
\end{theorem}

Here, under structure attacks, we follow the rationale of most works achieving certification in discrete space~\cite{lee2019tight,bojchevski2020efficient,wang2021certified}: the probability of $\tilde{g}_{\bm{X}}$ will always be 1 (i.e., certification succeeds) if $\underline{P_{\tilde{g}_{\bm{X}}=1}} > 0.5$.
Meanwhile, taking the smallest $\epsilon_{\bm{X}}$ over all possible $\tilde{\epsilon_{\bm{X}}}$ will ensure the certification successful in the space of node attributes, since perturbations over node attributes  smaller than $\epsilon_{\bm{X}}$ will not change the output of $\tilde{g}_{\bm{X}}(\bm{A} \oplus \bm{\Gamma}_{\bm{A}}, \bm{X})$ given any $\bm{\Gamma}_{\bm{A}} \in \bar{\mathcal{A}}$. We provide detailed proofs in Appendix in the online version.

\subsection{Certification in Practice}
\label{certification_practice}

We now introduce how to bridge the gaps between theoretical design and practical feasibility. 
We show the algorithmic routine of the proposed ELEGANT in Algorithm~\ref{elegant}.

\noindent \textbf{Estimating the Predicted Label Probabilities.}
According to~\cref{a_smooth}, it is necessary to obtain $\mathrm{Pr}(\tilde{g}_{\bm{X}}(\bm{A} \oplus \bm{\Gamma}_{\bm{A}}, \bm{X}) = c)$ ($c \in \{0, 1\}$) to determine the output of classifier $\tilde{g}_{\bm{X}}$.
We propose to leverage a Monte Carlo method to estimate such a probability. Specifically, we first randomly pick $N$ samples of $\bm{\Gamma}_{\bm{A}}$ as $\bar{\mathcal{A}}'$ ($\bar{\mathcal{A}}' \subset \bar{\mathcal{A}}$). Considering the output of $\tilde{g}_{\bm{X}}$ is binary, we then follow a common strategy~\cite{cohen2019certified} to consider this problem as a parameter estimation of a Binomial distribution:
we first count the number of returned label 1 and 0 under noise as $N_1$ and $N_0$ ($N_1 + N_0 = N$); then we choose a confidence level $1- \alpha$ and take the $\alpha$-th quantile of the beta distribution with parameters $N_1$ and $N_0$ as the estimated probability lower bound for returning label $c = 1$. We proved that all analysis remains true for such an estimation in Appendix.
Similar strategy is used to estimate the probability lower bound of yielding 1 for $g(\bm{A}, \bm{X} + \bm{\Gamma}_{\bm{X}})$.

  \begin{algorithm}[t]
    \small
    \vspace{1mm}
    \caption{Certified Defense on the Fairness of GNNs} \label{elegant}
    \begin{algorithmic}[1]
      \REQUIRE ~~\\ 
$\mathcal{G}$: graph data with potential malicious attacks; 
$f_{\bm{\theta}^*}$: an optimized GNN node classifier; 
$\mathcal{V}_{\text{train}}$, $\mathcal{V}_{\text{validation}}$, $\mathcal{V}_{\text{test}} \in \mathcal{V}$: the node set for training, validation, and test, respectively;
$\mathcal{V}_{\text{vul}} \in \mathcal{V}_{\text{test}}$: the set of vulnerable nodes that may bear attacks (on node attributes and/or graph topology);
$N_1$, $N_2$: sample size for the set of Bernoulli and Gaussian noise, respectively;
$\eta$: a given threshold for the exhibited bias;
$\alpha$: the parameter to indicate the confidence level ($1 - \alpha$) of the estimation;
$\sigma$: the std of the added Gaussian noise;
$\beta$: the probability of returning zero of the added Bernoulli noise;
\ENSURE ~~\\
$\epsilon_{\bm{A}}$: the certified defense budget over the adjacency matrix $\bm{A}$; 
$\epsilon_{\bm{X}}$: the certified defense budget over the node attribute matrix $\bm{X}$; 
$\hat{\bm{Y}}$: the output node classification results from the certified classifier; \\
\STATE Sample a set of Bernoulli noise $\mathcal{Q}_{\text{B}}$ containing $N_1$ samples;
\STATE Sample a set of Gaussian noise $\mathcal{Q}_{\text{G}}$ containing $N_2$ samples;
\FOR{$\omega_{\bm{A}}$ $\in$ $\mathcal{Q}_{\text{B}}$}
\FOR{$\omega_{\bm{X}}$ $\in$ $\mathcal{Q}_{\text{G}}$}
\STATE Calculate and collect the output of $f_{\bm{\theta}^*}$ under $\omega_{\bm{A}}$ and $\omega_{\bm{X}}$;
\STATE Calculate and collect the output of $g$ based on the output of $f_{\bm{\theta}^*}$;
\ENDFOR
\STATE Under $\mathcal{Q}_{\text{G}}$, collect the number of $g$ returning 1 and 0 as $n_1$ and $n_0$, respectively;
\STATE Estimate the lower bound of returning $c$ as \underline{$P_{g=c}$} determined by the larger one between $n_1$ and $n_0$;
  \IF{$n_1 > n_0$ and \underline{$P_{g=1}$} is larger than 0.5 with a confidence level larger than $1 - \alpha$ \textbf{or} $n_1 < n_0$ and \underline{$P_{g=0}$} is larger than 0.5 with a confidence level larger than $1 - \alpha$}
    \STATE Calculate and collect the value of $\tilde{\epsilon}_{\bm{X}}$;
  \ELSE
    \STATE \textbf{return} ABSTAIN
  \ENDIF
\ENDFOR
\STATE Collect the number of cases where $n_1 > n_0$ and estimate the lower bound of returning 1 as \underline{$P_{\tilde{g}_{\bm{X}}=1}$};
\IF{\underline{$P_{\tilde{g}_{\bm{X}}=1}$} is larger than 0.5 with a confidence level larger than $1 - \alpha$}
\STATE Calculate $\epsilon_{\bm{X}}$ (out of all $\tilde{\epsilon}_{\bm{X}}$) and $\epsilon_{\bm{A}}$ (based on the estimated \underline{$P_{\tilde{g}_{\bm{X}}=1}$}); 
\STATE Find $\hat{\bm{Y}}$ out of the collected output of $f_{\bm{\theta}^*}$;
\STATE \textbf{return} $\hat{\bm{Y}}$, $\epsilon_{\bm{X}}$, and $\epsilon_{\bm{A}}$;
\ELSE
\STATE \textbf{return} ABSTAIN
\ENDIF
    \end{algorithmic}
  \end{algorithm}

\noindent \textbf{Obtaining Fair Classification Results.} After achieving certified fairness defense based on $\tilde{g}_{\bm{A}, \bm{X}}$, we also need to obtain the corresponding node classification results (given by $f_{\bm{\theta}^*}$) over $\mathcal{V}_{\text{tst}}$. We propose to collect all classification results associated with the sampled $\bm{\Gamma}_{\bm{A}}' \in \bar{\mathcal{A}}'$ that leads to an estimated lower bound of $\mathrm{Pr}(\tilde{g}_{\bm{X}}(\bm{A} \oplus \bm{\Gamma}_{\bm{A}}', \bm{X}) = 1)$ to be larger than 0.5 as $\hat{\mathcal{Y}}'$.
Here $\hat{\mathcal{Y}}'$ is a set of output matrices of $f_{\bm{\theta}^*}$, where each matrix consists of the one-hot output classification results (as each row in the matrix) for all nodes.
We propose to take $\mathrm{argmin}_{\hat{\bm{Y}}'} \pi(\hat{\bm{Y}}', \mathcal{V}_{\text{tst}}), s.t. \; \hat{\bm{Y}}' \in \hat{\mathcal{Y}}'$ as the final node classification results.
Consider $\mathrm{Pr}(\tilde{g}_{\bm{X}}(\bm{A} \oplus \bm{\Gamma}_{\bm{A}}', \bm{X}) = 1)$ falls into the confidence interval characterized by $1- \alpha$, we have the probabilistic theoretical guarantee below.
\begin{proposition}
\label{fair_predictions}
(Probabilistic Guarantee for the Fairness Level of Node Classification).
For $\hat{\bm{Y}} = \mathrm{argmin}_{\hat{\bm{Y}}'} \pi(\hat{\bm{Y}}', \mathcal{V}_{\text{tst}}), s.t. \; \hat{\bm{Y}}' \in \hat{\mathcal{Y}}'$, we have $\mathrm{Pr}(\pi(\hat{\bm{Y}}, \mathcal{V}_{\text{tst}}) > \eta) < 0.5^{|\hat{\mathcal{Y}}'|}$.
\end{proposition}
Note that for a large enough sample size $N$, the cardinality of $\hat{\mathcal{Y}}'$ also tends to be large in practice. Hence it is safe to argue that $\mathrm{Pr}(\pi(\hat{\bm{Y}}, \mathcal{V}_{\text{tst}}) > \eta)$ tends to be small enough. In other words, we have a probability that is large enough to obtain results with a bias level lower than threshold $\eta$.

\noindent \textbf{Calculation of Perturbation Budgets.} 
We calculate $\epsilon_{\bm{A}}$ with the algorithm proposed in~\cite{wang2021certified}. 
For $\epsilon_{\bm{X}}$, we utilize a Monte Carlo procedure to empirically verify the certified defense budget over node attributes. 
Specifically, due to the intractability of enumerating all possible structure perturbations, we randomly sample a finite set $\bar{\mathcal{A}}'$ of $\bm{\Gamma}_{\bm{A}}$ and evaluate the certification condition under each sampled case. 
If any sampled $\bm{\Gamma}_{\bm{A}}'$ violates the certification condition, the procedure immediately abstains. 
Otherwise, we estimate the value of $\epsilon_{\bm{X}}$ with
$\min\{\tilde{\epsilon}_{\bm{X}}: \tilde{\epsilon}_{\bm{X}} \text{ is derived with classifier } \tilde{g}_{\bm{X}}(\bm{A} \oplus \bm{\Gamma}_{\bm{A}}', \bm{X}), \bm{\Gamma}_{\bm{A}}' \in \bar{\mathcal{A}}'\}$.

\section{Experimental Evaluations}

In this section, we aim to answer three research questions: \textbf{RQ1}: How well does ELEGANT perform in achieving certified fairness defense? 
\textbf{RQ2}: How does ELEGANT perform under fairness attacks compared to other popular fairness-aware GNNs?
\textbf{RQ3}: How does ELEGANT perform under different settings of parameters?
We present the main experimental settings and representative results in this section due to space limits. Detailed settings and supplementary experiments are in Appendix of the online version.

\subsection{Experimental Settings}
\label{settings}

\noindent \textbf{Downstream Task and Datasets.}
We focus on the widely studied node classification task, which is one of the most representative tasks in the domain of learning on graphs.
We adopt three real-world network datasets that are widely used to perform studies on the fairness of GNNs, namely German Credit~\cite{agarwal2021towards,dua2017uci}, Recidivism~\cite{agarwal2021towards,jordan2015effect}, and Credit Defaulter~\cite{agarwal2021towards,yeh2009comparisons}.
(1) \textit{German Credit.} Each node is a client in a German bank~\cite{dua2017uci}, while each edge between any two clients represents that they bear similar credit accounts. Here the gender of bank clients is considered as the sensitive attribute, and the task is to classify the credit risk of the clients as high or low.
(2) \textit{Recidivism.} Each node denotes a defendant released on bail at the U.S state courts during 1990-2009~\cite{jordan2015effect}, and defendants are connected based on the similarity of their past criminal records and demographics. Here the race of defendants is considered as the sensitive attribute, and the task is to classify defendants into more likely vs. less likely to commit a violent crime after being released.
(3) \textit{Credit Defaulter.} This dataset contains credit card users collected from financial agencies~\cite{yeh2009comparisons}. Specifically, each node in this network denotes a credit card user, and users are connected based on their spending and payment patterns. The sensitive attribute is the age period of users, and the task is to predict the future default of credit card for these users. We present the statistics of the three datasets above in~\cref{datasets}.
For the three real-world datasets used in this paper, we adopt the split rate for the training set and validation set as 0.4 and 0.55, respectively. The input node features are normalized before they are fed into the GNNs and the corresponding explanation models. For the downstream task \textit{node classification}, only the labels of the nodes in the training set is available for all models during the training process. 
The trained GNN models with the best performance on the validation set are preserved for test and explanation.

\begin{table}[]
\vspace{2mm}
\caption{The statistics and basic information about the six real-world datasets adopted for experimental evaluation. Sens. represents the semantic meaning of sensitive attribute.}
\label{datasets}
\centering
\setlength{\tabcolsep}{2.12pt}
\renewcommand{\arraystretch}{1.1}
\begin{tabular}{lcccccc}
\toprule
\textbf{Dataset}             & \textbf{German Credit}        & \textbf{Recidivism}   & \textbf{Credit Defaulter}       \\
\midrule
\textbf{\# Nodes}               & 1,000                      & 18,876             & 30,000                    \\
\textbf{\# Edges}              & 22,242                     & 321,308            & 1,436,858               \\
\textbf{\# Attributes}      & 27                         & 18                 & 13                         \\
\textbf{Avg. degree}         & 44.5                      & 34.0              & 95.8                      \\
\textbf{Sens.}        & Gender       & Race  & Age     \\
\textbf{Label}       & Credit status & Bail decision   & Future default   \\
\bottomrule
\end{tabular}
\end{table}

\noindent \textbf{Evaluation Metrics.} We perform evaluation from three main perspectives, including model utility, fairness, and certified defense.
To evaluate utility, we adopt the node classification accuracy. To evaluate fairness, we adopt the widely used metrics $\Delta_{\text{SP}}$ (measuring bias under \textit{Statistical Parity}) and $\Delta_{\text{EO}}$ (measuring bias under \textit{Equal Opportunity}).
To evaluate certified defense, we extend a traditional metric named \textit{Certified Accuracy}~\cite{wang2021certified,cohen2019certified} in our experiments, and we name it as \textit{Fairness Certification Rate} (FCR).
Specifically, existing GNN certification works mainly focus on a certain individual node, and utilize certified accuracy to measure the ratio of nodes that are correctly classified and also successfully certified out of all test nodes~\cite{wang2021certified}. In this paper, however, we perform certified (fairness) defense for individuals over an entire test set (instead of for any specific individual). Accordingly, we propose to sample multiple test sets out of nodes that are not involved in the training and validation set. Then we perform certified fairness defense for all sampled test sets, and utilize the ratio of test sets that are successfully certified over all sampled sets as the metric of certified defense. Here, FCR aims to use Monte Carlo method to estimate the probability of being successfully certified for a randomly sampled test node set.

\begin{table*}[t]
\setlength{\tabcolsep}{6.52pt}
\renewcommand{\arraystretch}{1.1}
\centering
\vspace{1mm}
\caption{Comparison between vanilla GNNs and certified GNNs under ELEGANT over three popular GNNs across three real-world datasets. Here ACC denotes node classification accuracy, and E- prefix marks out the GNNs under ELEGANT with certification. $\uparrow$ denotes the larger the value of this metric is, the better; $\downarrow$ represents the opposite. All numerical values in the table are in percentage, and the best ones are marked out in bold.}
\label{performance}
\begin{tabular}{cccccccccccc}
\hline
\hline
            & \multicolumn{3}{c}{\textbf{German Credit}} &  & \multicolumn{3}{c}{\textbf{Recidivism}} &  & \multicolumn{3}{c}{\textbf{Credit Defaulter}} \\
            \cline{2-4}  \cline{6-8}  \cline{10-12}
            & \textbf{ACC ($\uparrow$)}   & \textbf{Bias ($\downarrow$)}   & \textbf{FCR ($\uparrow$)}   &  & \textbf{ACC ($\uparrow$)}   & \textbf{Bias ($\downarrow$)}   & \textbf{FCR ($\uparrow$)}   &  & \textbf{ACC ($\uparrow$)}   & \textbf{Bias ($\downarrow$)}   & \textbf{FCR ($\uparrow$)}   \\
            \hline
            \hline
\textbf{SAGE}         &67.3 $_{\pm 2.14}$       & 50.6 $_{\pm 15.9}$       & N/A    &  & 89.8 $_{\pm 0.66}$      & 9.36 $_{\pm 3.15}$       & N/A      &  & \textbf{75.9 $_{\pm 2.18}$}      &  13.0 $_{\pm 4.01}$       & N/A      \\
\hline
\textbf{E-SAGE} &\textbf{71.0 $_{\pm 1.27}$}  & \textbf{16.3 $_{\pm 10.9}$}       & 98.7 $_{\pm 1.89}$     &  & \textbf{89.9 $_{\pm 0.90}$}      &\textbf{6.39 $_{\pm 2.85}$}        & 94.3 $_{\pm 6.65}$      &  & 73.4 $_{\pm 0.50}$      & \textbf{8.94 $_{\pm 0.99}$}       & 94.3 $_{\pm 3.30}$      \\
\hline
\hline
\textbf{GCN}         & \textbf{59.6 $_{\pm 3.64}$}      & 37.4 $_{\pm 3.24}$       & N/A      &  & \textbf{90.5 $_{\pm 0.73}$}      & 10.1 $_{\pm 3.01}$       & N/A      &  & \textbf{65.8 $_{\pm 0.29}$}      & 11.1 $_{\pm 3.22}$       & N/A      \\
\hline
\textbf{E-GCN} & 58.2 $_{\pm 1.82}$      &\textbf{3.52 $_{\pm 3.77}$}        & 96.3 $_{\pm 1.89}$      &  & 89.6 $_{\pm 0.74}$      & \textbf{9.56 $_{\pm 3.22}$}       & 96.0 $_{\pm 3.56}$      &  & 65.2 $_{\pm 0.99}$      &\textbf{7.28 $_{\pm 1.46}$}        & 92.7 $_{\pm 5.19}$      \\
\hline
\hline
\textbf{JK}         & \textbf{63.3 $_{\pm 4.11}$}      & 41.2 $_{\pm 18.1}$       & N/A      &  & \textbf{91.9 $_{\pm 0.54}$}      & 10.1 $_{\pm 3.15}$       & N/A      &  & 76.6 $_{\pm 0.69}$      & 9.24 $_{\pm 0.60}$       & N/A      \\
\hline
\textbf{E-JK} & 62.3 $_{\pm 4.07}$      & \textbf{22.4 $_{\pm 1.95}$}       & 97.0 $_{\pm 3.00}$      &  &89.3 $_{\pm 0.33}$       & \textbf{6.26 $_{\pm 2.78}$}       & 89.5 $_{\pm 10.5}$      &  &\textbf{77.7 $_{\pm 0.27}$}       & \textbf{3.37 $_{\pm 2.64}$}       & 99.3 $_{\pm 0.47}$      \\
\hline
\hline
\end{tabular}
\end{table*}

\noindent \textbf{GNN Backbones and Baselines.} Note that ELEGANT serves as a plug-and-play framework for any optimized GNNs ready to be deployed.
To evaluate the generality of ELEGANT across GNNs, we adopt three of the most representative GNNs spanning across simple and complex ones, namely Graph Sample and Aggregate Networks~\cite{hamilton2017inductive} (GraphSAGE), Graph Convolutional Networks~\cite{DBLP:conf/iclr/KipfW17} (GCN), and Jumping Knowledge Networks (JK).
Note that to the best of our knowledge, existing works on fairness certification cannot certify the attacks over two data modalities (i.e., continuous node attributes and binary graph topology) at the same time, and thus cannot be naively generalized onto GNNs.
Hence we compare the usability of GNNs before and after certification with ELEGANT.
Moreover, we also adopt two popular fairness-aware GNNs as baselines to evaluate bias mitigation, including FairGNN~\cite{dai2021say} and NIFTY~\cite{agarwal2021towards}. 
(1) \textit{FairGNN.} For FairGNN, we adopt the official implementations from~\cite{dai2021say}. Hyper-parameters corresponding to the GNN model structure (such as the number of hidden dimensions) are ensured to be the same as the vanilla GNNs for a fair comparison. Other parameters are carefully tuned under the guidance of the recommended training settings.
(2) \textit{NIFTY.}
For NIFTY, we use the official implementations given by~\cite{agarwal2021towards}. We ensured that the parameters related to the GNN model structure stay the same as the original GNNs for a fair comparison. We also adjust other parameters based on the suggested settings for better performance.

\noindent \textbf{Selection of $\epsilon$ and $\beta$.} 
There are two critical parameters, $\epsilon$ and $\beta$, that could affect the effectiveness of ELEGANT. These two parameters control the level of randomness for the added Gaussian and Bernoulli noise, respectively. 
Intuitively, larger $\epsilon$ and $\beta$ will induce more randomness in the node attributes and graph structure, which could make ELEGANT more robust to perturbations with larger sizes and thus achieve larger $\epsilon_{\bm{X}}$ and $\epsilon_{\bm{A}}$. However, if $\epsilon_{\bm{X}}$ and $\epsilon_{\bm{A}}$ are too large, the randomness could go beyond what the GNN classifier can manage and could finally cause failure in certification. Hence it is necessary to first determine appropriate values of $\epsilon_{\bm{X}}$ and $\epsilon_{\bm{A}}$ for ELEGANT.
Here we propose a strategy for parameter selection to realize as large certified defense budgets as possible.
Specifically, we first set an empirical $\eta$ to be 25\% higher than the fairness level of the corresponding vanilla GNN model. Such a threshold calibrates across different GNNs and can be considered as a reasonable threshold for the exhibited bias.
Then we determine two wide search spaces for $\sigma$ and $\beta$, respectively, and compute the averaged $\epsilon_{\bm{X}}$ and $\epsilon_{\bm{A}}$ from multiple runs over each pair of $\sigma$ and $\beta$ values.
We now rank $(\sigma, \beta)$ pairs based on the averaged $\epsilon_{\bm{X}}$ and $\epsilon_{\bm{A}}$ in a descending order, respectively.
Finally, we truncate the obtained two rankings from their most top-ranked $(\sigma, \beta)$ pair to the tail, until the two truncated rankings have the first overlapped $(\sigma, \beta)$ pair.
Such an identified $(\sigma, \beta)$ pair can achieve large and balanced certification budgets over both $\bm{A}$ and $\bm{X}$, and hence they are recommended.

\noindent \textbf{Threat Models.} 
We propose to evaluate the performance of ELEGANT and other fairness-aware GNN models under actual attacks on fairness.
We first introduce the threat model over graph structure. To the best of our knowledge, FA-GNN~\cite{hussain2022adversarial} is the only work that performs graph structure attacks targeting the fairness of GNNs.
Hence we adopt FA-GNN to attack graph structure.
In terms of node attributes, to the best of our knowledge, no existing work has made any explorations. Hence we directly utilize gradient ascent to perform attacks. Specifically, after structure attacks have been performed, we identify the top-ranked node attribute elements (out of the node attribute matrix) that positively influence the exhibited bias the most via gradient ascent. For any given budget (of attacks) on node attributes, we add perturbations to these elements in proportion to their gradients.

\subsection{RQ1: Fairness Certification Effectiveness}
To answer RQ1, we investigate the performance of different GNNs after certification across different real-world attributed network datasets over FCR, utility, and fairness. 
We present the experimental results across three GNN backbones and three real-world attributed network datasets in~\cref{performance}. Here bias is measured with $\Delta_{\text{SP}}$, and we have similar observations on $\Delta_{\text{EO}}$.
We summarize the main observations as follows: 
(1) \textbf{Fairness Certification Rate (FCR).}
We observe that ELEGANT realizes values of FCR around or even higher than 90\% for all three GNN backbones and three attributed network datasets, especially for the German Credit dataset, where vanilla GNNs tend to exhibit a high level of bias. The corresponding intuition is that, for nodes in any randomly sampled test set, we have a probability around or higher than 90\% to successfully certify the fairness level of the predictions yielded by the GNN model with our proposed framework ELEGANT. Hence ELEGANT achieves a satisfying fairness certification rate across all adopted GNN backbones and datasets.
(2) \textbf{Utility.}
We found that compared with those vanilla GNN backbones, certified GNNs with ELEGANT also exhibit comparable and even higher node classification accuracy values in all cases. Hence we conclude that our proposed framework ELEGANT does not significantly jeopardize the utility of the vanilla GNN models, and those certified GNNs with ELEGANT still bear a high level of usability in terms of node classification accuracy.
(3) \textbf{Fairness.}
Although the goal of ELEGANT is not debiasing GNNs, we observe that certified GNNs with ELEGANT achieve better performances in all cases in terms of algorithmic fairness compared with those vanilla GNNs. This demonstrates that the proposed framework ELEGANT also contributes to bias mitigation. We conjecture that such an advantage of debiasing could be a mixed result of (1) adding random noise on node attributes and graph topology (as in \cref{noise1} and \cref{noise2}) and (2) the proposed strategy of obtaining fair classification results (as in \cref{certification_practice}).

\subsection{RQ2: Fairness Certification under Attacks}
To answer RQ2, we perform attacks on the fairness of GCN, E-GCN, FairGNN (with a GCN backbone), and NIFTY (with a GCN backbone). Considering the large size of the quadratic space spanned by the sizes of perturbations $\bm{\Delta}_{\bm{A}}$ and $\bm{\Delta}_{\bm{X}}$, we present the evaluation under four representative ($\|\bm{\Delta}_{\bm{A}}\|_{0}$, $\|\bm{\Delta}_{\bm{X}}\|_F$) pairs. We set the threshold for bias $\eta$ to be 50\% higher than the fairness level of the vanilla GCN model on clean data, since it empirically helps to achieve a high certification success rate under large perturbations.

We present the utility and the fairness levels of the four models in terms of $\Delta_{\text{EO}}$ in~\cref{acc_under_attacks} and~\cref{under_attacks}, respectively. 
Note that ELEGANT determines the final output by examining the exhibited bias across the outputs obtained from Monte Carlo sampling, as discussed in~\cref{certification_practice}. 
As a result, evaluating fairness metrics such as $\Delta_{\text{EO}}$ requires explicit class predictions for test nodes under the certified setting. 
To this end, we utilize a vanilla GCN to predict the labels for test nodes and use these predictions to compute the corresponding fairness metrics. 
This evaluation protocol follows common practice in fairness assessment and allows us to consistently compare different models under fairness attacks. 
We summarize the main observations as follows, and similar trends are also observed across other GNN backbones and datasets.
%
(1) \textbf{Utility.} We first focus on the comparison over the utility under attacks, where we utilize node classification accuracy as the indicator of model utility. The fairness-aware GNNs are found to exhibit better utility compared with the vanilla GNNs, which is a common observation consistent with a series of existing works~\cite{agarwal2021towards,dong2021edits}. More importantly, we observe that the ELEGANT does not jeopardize the performance of GNN compared with the utility of the vanilla GNN. This demonstrates a high level of usability for ELEGANT in real-world applications.
(2) \textbf{Fairness.} We found that the GCN model with the proposed framework ELEGANT achieves the lowest level of bias in all cases of fairness attacks. This observation is consistent with the superiority in fairness found in~\cref{performance}, which demonstrates that the fairness superiority of ELEGANT maintains even under attacks within a wide range of attacking perturbation sizes.
(3) \textbf{Certification on Fairness.} We now compare the performance of E-GCN across different attacking perturbation sizes. We observed that under relatively small attacking perturbation sizes, i.e., ($2^0, 10^{-1}$), ($2^1, 10^{0}$), and ($2^2, 10^{1}$), ELEGANT successfully achieves certification over fairness, and the bias level increases slowly as the size of attacks increases.
Under relatively large attacking perturbation size, i.e., ($2^3, 10^{2}$), although the attacking budgets go beyond the certified budgets, GCN under ELEGANT still exhibits a fairness level far lower than the given bias threshold $\eta$, and the fairness superiority maintains.
Therefore, we draw the conclusion that the adopted estimation strategies are safe towards the goal of achieving fairness certification.

\begin{figure}[t]
    \centering
\includegraphics[width=0.9\linewidth]{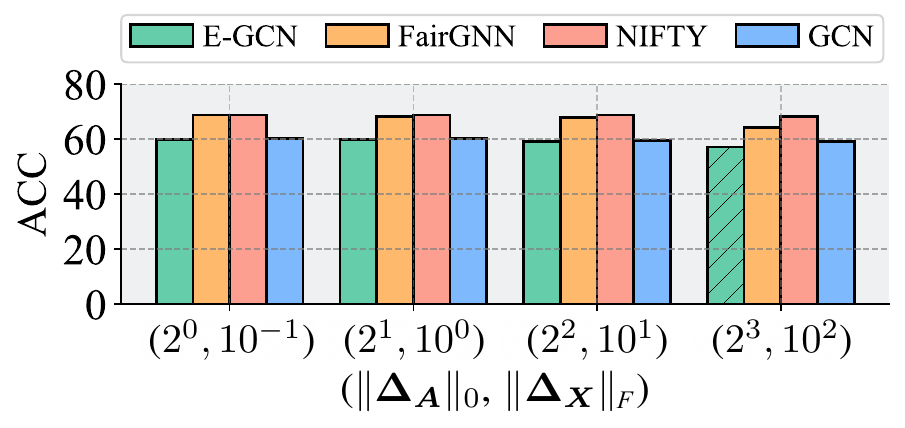}
    \caption{The utility of GCN, E-GCN, FairGNN, and NIFTY under fairness attacks on German Credit. The shaded bar indicates that certified budget $\epsilon_{\bm{A}} \leq \|\bm{\Delta}_{\bm{A}}\|_{0}$ or $\epsilon_{\bm{X}} \leq  \|\bm{\Delta}_{\bm{X}}\|_{F}$.} 
    \label{acc_under_attacks}
\end{figure}

\begin{figure}[t]
\vspace{-3mm}
    \centering
\includegraphics[width=0.83\linewidth]{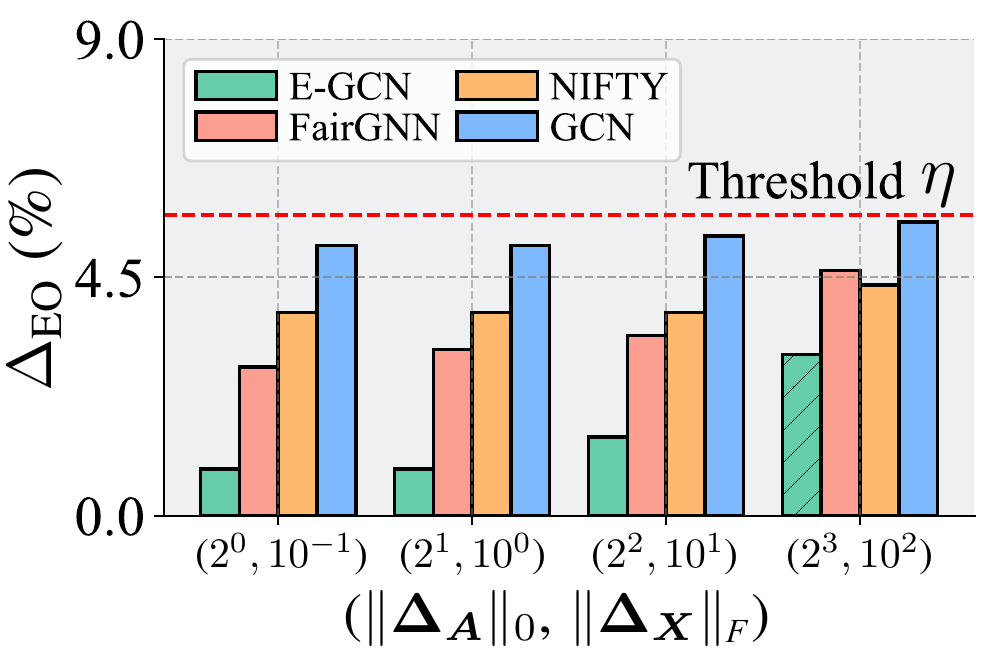}
    \caption{The bias levels of GCN, E-GCN, FairGNN, and NIFTY under fairness attacks on German Credit. The shaded bar indicates that certified budget $\epsilon_{\bm{A}} \leq \|\bm{\Delta}_{\bm{A}}\|_{0}$ or $\epsilon_{\bm{X}} \leq  \|\bm{\Delta}_{\bm{X}}\|_{F}$. The y-axis is in logarithmic scale for visualization purposes.} 
    \label{under_attacks}
\end{figure}

\begin{figure}[!t]
\centering
    \begin{subfigure}[t]{0.445\textwidth}
    \includegraphics[width=1.04\textwidth]{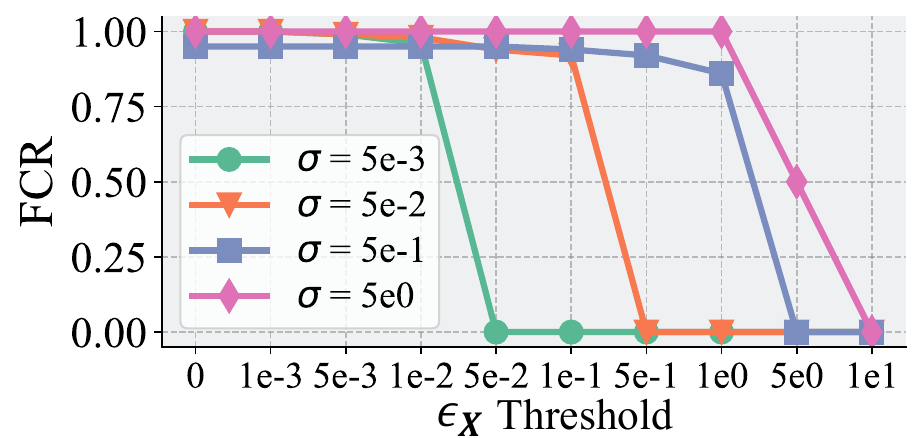}
        \caption[Network2]%
        {{FCR of certification for $\sigma$ over node attributes}} 
        \label{param1}
    \end{subfigure}
    
    \vspace{3mm}
    \begin{subfigure}[t]{0.445\textwidth}
    \includegraphics[width=1.04\textwidth]{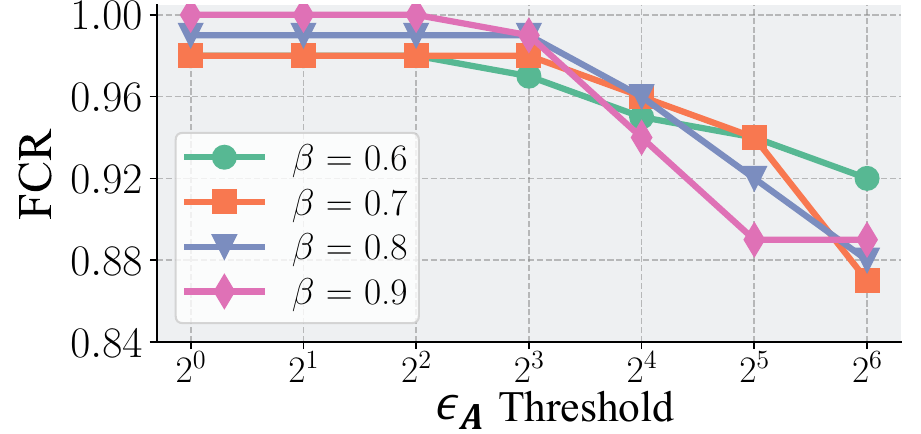}
        \caption[Network2]%
        {{FCR of certification for $\beta$ over graph topology}}    
        \label{param2}
    \end{subfigure} 
\caption{Parameter study of $\sigma$ over $\epsilon_{\bm{X}}$ (a) and $\beta$ over $\epsilon_{\bm{A}}$ (b). Experimental results are presented based on GCN over German credit and Credit Defaulter for (a) and (b), respectively. Similar tendencies of FCR value w.r.t. the threshold values can also be observed based on other GNNs and datasets.}
\label{effectiveness-2}
\end{figure}

\subsection{RQ3: Parameter Study}
\label{param_study}
To answer RQ3, we propose to perform parameter study focusing on two most critical parameters, $\sigma$ and $\beta$. 
To examine how $\sigma$ and $\beta$ influence the effectiveness of ELEGANT in terms of both FCR and certified defense budgets, we set numerical ranges for $\epsilon_{\bm{X}}$ (from 0 to 1e1) and $\epsilon_{\bm{A}}$ (from 0 to $2^6$) and divide the two ranges into grids. 
In both ranges, we consider the dividing values of the grids as thresholds for certification budgets.
In other words, under each threshold, we only consider the test sets with the corresponding certified defense budget being larger than this threshold as successfully certified ones, and the values of FCR are re-computed accordingly.
Our rationale here is that with the thresholds (for $\epsilon_{\bm{X}}$ and $\epsilon_{\bm{A}}$) increasing, if FCR reduces slowly, this demonstrates that most successfully certified test sets are associated with large certified defense budgets. However, if FCR reduces fast, then most successfully certified test sets only bear small certified defense budgets. We present other detailed settings in Appendix of the online version.

Here we present the experimental results of $\sigma$ and $\beta$ on the basis of GCN model using German Credit dataset in~\cref{param1} and Credit Defaulter in~\cref{param2}, respectively. We also have similar observations on other GNNs and datasets.
We summarize the main observations as follows: 
(1) \textbf{Analysis on $\sigma$.}
We observe that most cases with larger $\sigma$ are associated with a larger FCR compared with the cases where $\sigma$ is relatively small. In other words, larger values of $\sigma$ typically make FCR reduce slower w.r.t. the increasing of $\epsilon_{\bm{X}}$ threshold. This indicates that increasing the value of $\sigma$ helps realize larger certified defense budgets on node attributes, i.e., the increase of $\sigma$ dominates the tendency of $\epsilon_{\bm{X}}$ given in~\cref{x_r}.
Nevertheless, it is worth mentioning that if $\sigma$ is too large, the information encoded in the node attributes could be swamped by the Gaussian noise and finally corrupt the classification accuracy. Hence moderately large values for $\sigma$, e.g., 5e-1 and 5e0, are recommended.
(2) \textbf{Analysis on $\beta$.}
We found that (1) for cases with relatively large $\beta$ (e.g., 0.8 and 0.9), the FCR also tends to be larger (compared with cases where $\beta$ is smaller) at $\epsilon_{\bm{A}}$ threshold being 0. 
Such a tendency is reasonable, since in these cases, the expected magnitude of the added Bernoulli noise is small. Correspondingly, GNNs under ELEGANT perform similarly to vanilla GNNs, and thus an $\eta$ larger than the bias level of vanilla GNNs is easier to be satisfied (compared with cases under smaller values of $\beta$); (2) for cases with relatively large $\beta$, the value of FCR generally reduces faster (w.r.t. $\epsilon_{\bm{A}}$ threshold) than cases where $\beta$ is smaller. 
Therefore, we recommend that for any test set of nodes: (1) if the primary goal is to achieve certification with a high probability, then larger values for $\beta$ (e.g., 0.8 and 0.9) would be preferred; (2) if the goal is to achieve certification with larger budgets on the graph topology, smaller values for $\beta$ (e.g., 0.6 and 0.7) should be selected.

\section{Related Work}

\noindent \textbf{Algorithmic Fairness of GNNs.}
Despite the success of GNNs in a plethora of downstream tasks, most GNNs do not have fairness considerations. As a consequence, how to equip GNNs with fairness consideration has attracted much research attention in recent years~\cite{wang2022improving,zhao2025fairness,zhang2024trustworthy}. In general, 
existing GNN works focusing on fulfilling algorithmic fairness mainly focus on group fairness and individual fairness~\cite{dong2022fairness}. Specifically, group fairness requires that each demographic subgroup (divided by sensitive attributes such as gender and race) in the graph should have their fair share of interest based on predictions~\cite{mehrabi2021survey}. Adversarial training is among the most popular strategies~\cite{dai2021say,dong2022fairness}. 
Its core idea is to train an adversary, such that the learned embeddings can be decoupled from sensitive information during the min-max game between the adversary and the GNN model.
In addition, regularization~\cite{agarwal2021towards,fan2021fair,zhang2021multi}, topology modification~\cite{dong2021edits,spinelli2021biased}, and orthogonal projection~\cite{palowitchdebiasing} are also commonly used strategies. On the other hand, individual fairness it requires that similar individuals should be treated similarly~\cite{DBLP:conf/innovations/DworkHPRZ12}, where such similarity may be determined in different ways~\cite{KangHMT20,dong2021individual}. Designing optimization regularization terms to promote individual fairness for GNNs is a common strategy~\cite{fan2021fair,dong2021individual,song2022guide}. The success of these explorations has been evidenced by various real-world applications of fairness-aware GNNs~\cite{chen2024fairness,zhang2024towards,xia2025fairtp,jin2023survey,zhang2024understanding,yoo2024ensuring,abdelrazek2023fairup}.
Nevertheless, despite the research advancements in the field of algorithmic fairness on GNNs, the adversarial defense against fairness attacks still remains in its infancy and has not been thoroughly explored. To the best of our knowledge, our paper serves as the first comprehensive study dedicated to addressing this important research problem, paving the way for future investigations in this under-explored area.

\noindent 
\textbf{Certified Defense in Discrete Input Spaces.}
Recent studies have proposed certified defenses that guarantee the same predicted label on clean and perturbed input data under limited perturbation budgets. 
One prominent approach is randomized smoothing~\cite{cohen2019certified}, which enhances the robustness of classifiers by introducing Gaussian noise. 
However, in discrete input spaces such as graph structures, Gaussian noise can lead to loose certification bounds~\cite{lee2019tight}.
One prominent approach for certified defense is randomized smoothing~\cite{cohen2019certified} where Gaussian noise is added to improve the robustness of the smoothed classifier.
However, when tackling the discrete input space such as graph structures, adding Gaussian noise can lead to loose certification bound~\cite{lee2019tight}.
To address this issue, a pioneering method achieving certified defense in discrete input spaces (i.e., point-wise certification)~\cite{lee2019tight} was proposed for binary classification tasks, aiming to tighten the certification bound in discrete spaces. 
It partitions the discrete input space into chunks with identical noise ratios. 
By accumulating the output probabilities across these chunks, a tight lower bound of the perturbed output probability can be achieved. 
If this lower bound exceeds $0.5$, the certification region can be properly characterized.
Building upon this, subsequent studies have extended point-wise certification to different problem settings~\cite{jia2020certified}. 
One approach~\cite{bojchevski2020efficient} focused on binary graph structures and introduced a data-dependent Bernoulli noise, which reduces the number of partitioned chunks. 
Given the sparsity of graph structures, this method improves the efficiency of point-wise certification for GNNs~\cite{lee2019tight}. 
Another extension~\cite{wang2021certified} adapted point-wise certification to multi-class classification.
Different from these methods, our structure certification extends the approach in~\cite{bojchevski2020efficient} with data-independent noise. In this way, such noise can be generated before the graph data is given, and thus it facilitates a composition of certification for graph structure and node attributes.

\section{Conclusion}

In this paper, we take initial steps to tackle a novel problem of certifying GNN node classifiers on their fairness levels. 
Specifically, we aim to achieve a classifier on top of any optimized GNN node classifier associated with certain perturbation budgets, such that it is impossible for attackers to corrupt the fairness level of GNNs within such budgets. To address this problem, we propose a principled framework, ELEGANT, which achieves certification on top of any optimized GNN node classifier associated with certain perturbation budgets, such that it is impossible for attackers to corrupt the fairness level of predictions within such budgets. Notably, ELEGANT is designed to serve as a plug-and-play framework for any optimized GNNs ready to be deployed and does not rely on any assumption over GNN structure or parameters. Extensive experiments verify the satisfying effectiveness of ELEGANT. In addition, we also found ELEGANT beneficial to GNN debiasing, and explored how its parameters influence the certification performance. We leave certifying the fairness level of GNNs over other learning tasks on graphs as future works.

\section{Acknowledgements}

This work is supported in part by the National Science Foundation (NSF) under grants IIS-2144209, IIS-2223769, CNS-2154962, BCS-2228534, IIS-2416070, and CMMI-2411248; the Office of Naval Research (ONR) under grant N000142412636; and the Commonwealth Cyber Initiative (CCI) under grant VV-1Q24-011, and the gift funding from Netflix and Snap.

\bibliographystyle{ACM-Reference-Format}
\balance
\bibliography{ref}

\appendix


\allowdisplaybreaks

\section{Proofs}

\label{proofs}

For better clarity, for a matrix $\bm{X}$, we use $\bm{X}[i,j]$ to denote the element at the $i$-th row and the $j$-th column; for a vector $\bm{x}$, we use $\bm{x}[i]$ to denote its $i$-th component.

\subsection{Proof of Theorem 2}
\label{th4}


\begin{proof}
Below we discuss the proofs for the budget of structure perturbation and node attribute perturbation, respectively.

\noindent \textbf{Structure Perturbation.} 
The proofs from~\cite{wang2021certified} is directly applicable to the budget for structure perturbation here.

\noindent \textbf{Node Attribute Perturbation.} 
Recall that 
\begin{equation}
\begin{aligned}
&\tilde{g}_{\bm{A},\bm{X}}(\bm{A}, \bm{X}) = \mathrm{argmax}_{c \in \{0, 1\}} \mathrm{Pr}( \tilde{g}_{\bm{X}}(\bm{A}\oplus\bm{\Gamma}_{\bm{A}}, \bm{X}) = c) \; \text{and} \\ 
&\tilde{g}_{\bm{X}}(\bm{A}, \bm{X}) = \mathrm{argmax}_{c \in \{0, 1\}}\mathrm{Pr}( g(f_{\bm{\theta}^*}, \bm{A}, \bm{X} + \bm{\Gamma}_{\bm{X}}, \eta, \mathcal{V}_{\text{tst}}) = c),
\end{aligned}
\end{equation}
to certify the fairness level, we assume that $\tilde{g}_{\bm{A},\bm{X}}(\bm{A}, \bm{X})=1$, which means that
\begin{equation}
\mathrm{Pr}(\tilde{g}_{\bm{X}}(\bm{A}\oplus\bm{\Gamma}_{\bm{A}}, \bm{X}) = 1)>0.5.
\end{equation}
For any perturbation $\|\bm{\Delta}_{\bm{X}}\|_2\leq\epsilon_{\bm{X}}\leq\tilde{\epsilon_{\bm{X}}}$, we have $\tilde{g}_{\bm{X}}(\bm{A}\oplus\bm{\Gamma}_{\bm{A}}, \bm{X}+\bm{\Delta}_{\bm{X}}) = \tilde{g}_{\bm{X}}(\bm{A}\oplus\bm{\Gamma}_{\bm{A}}, \bm{X})$ for any $\bm{\Gamma}_{\bm{A}} \in \bar{\mathcal{A}}$ where $\tilde{\epsilon_{\bm{X}}}$ is derived with classifier $\tilde{g}_{\bm{X}}(\bm{A} \oplus \bm{\Gamma}_{\bm{A}}, \bm{X})$.
Regarding the randomness of $\bm{\Gamma}_{\bm{A}}$, we have
\begin{equation}\label{eq:sample space inequality}
\mathrm{Pr}(\tilde{g}_{\bm{X}}(\bm{A}\oplus\bm{\Gamma}_{\bm{A}}, \bm{X}+\bm{\Delta}_{\bm{X}}) = 1)=\mathrm{Pr}(\tilde{g}_{\bm{X}}(\bm{A}\oplus\bm{\Gamma}_{\bm{A}}, \bm{X}) = 1)>0.5.
\end{equation}
Hence we obtain that $\tilde{g}_{\bm{A},\bm{X}}(\bm{A}, \bm{X}+\bm{\Delta}_{\bm{X}})=\tilde{g}_{\bm{A},\bm{X}}(\bm{A}, \bm{X})$ for any perturbation $\|\bm{\Delta}_{\bm{X}}\|_2\leq\epsilon_{\bm{X}}$. 
\end{proof}

\subsection{Proof of Proposition 1}
\label{proposition1}


\begin{proof}
For the practical certification, we add perturbations within certified budgets and derive independent identically distributed output samples $\hat{\mathcal{Y}}'$ by Monte Carlo. For each sample $\hat{\bm{Y}}'\in\hat{\mathcal{Y}}'$, we have $\mathrm{Pr}(\pi(\hat{\bm{Y}}', \mathcal{V}_{\text{tst}}) > \eta)<0.5$ according to~\cref{x_r} and~\cref{a_r}. $\pi(\hat{\bm{Y}}, \mathcal{V}_{\text{tst}}) > \eta$ indicates that $\pi(\hat{\bm{Y}}', \mathcal{V}_{\text{tst}}) > \eta$ for any $\hat{\bm{Y}}'\in\hat{\mathcal{Y}}'$. Consequently, we have
\begin{equation}
\mathrm{Pr}(\pi(\hat{\bm{Y}}, \mathcal{V}_{\text{tst}}) > \eta)=\Pi_{\hat{\bm{Y}}'\in\hat{\mathcal{Y}}'}\mathrm{Pr}(\pi(\hat{\bm{Y}}', \mathcal{V}_{\text{tst}}) > \eta)<0.5^{|\hat{\mathcal{Y}}'|}.
\end{equation}
\end{proof}

\section{Reproducibility and Supplementary Analysis}

In this section, our primary emphasis is on ensuring the replicability of our experiments, which serves as an extension to Section 4. To begin with, we offer a comprehensive introduction of the three real-world datasets adopted in our experiments. Subsequently, we introduce the detailed experimental settings, as well as the implementation details of our proposed framework, ELEGANT, alongside GNNs and baseline models. Moreover, we outline those essential packages, including their versions, that were utilized in our experiments. Lastly, we elaborate on the supplementary analysis on the time complexity of ELEGANT. Open-source code is available at: \url{https://github.com/yushundong/ELEGANT}.

\subsection{Datasets}

%
%
%
In our experiments, we adopt three real-world network datasets that are widely used to perform studies on the fairness of GNNs, namely German Credit~\cite{dua2017uci,agarwal2021towards}, Recidivism~\cite{jordan2015effect,agarwal2021towards}, and Credit Defaulter~\cite{yeh2009comparisons,agarwal2021towards}.
We introduce their basic information below. 

\noindent
\textbf{(1) German Credit.} Each node is a client in a German bank~\cite{dua2017uci}, while each edge between any two clients represents that they bear similar credit accounts. Here the gender of bank clients is considered as the sensitive attribute, and the task is to classify the credit risk of the clients as high or low.

\noindent
\textbf{(2) Recidivism.} Each node denotes a defendant released on bail at the U.S state courts during 1990-2009~\cite{jordan2015effect}, and defendants are connected based on the similarity of their past criminal records and demographics. Here the race of defendants is considered as the sensitive attribute, and the task is to classify defendants into more likely vs. less likely to commit a violent crime after being released.

\noindent
\textbf{(3) Credit Defaulter.} This dataset contains credit card users collected from financial agencies~\cite{yeh2009comparisons}. Specifically, each node in this network denotes a credit card user, and users are connected based on their spending and payment patterns. The sensitive attribute is the age period of users, and the task is to predict the future default of credit card for these users. We present the statistics pf the three datasets above in~\cref{datasets}.

For the three real-world datasets used in this paper, we adopt the split rate for the training set and validation set as 0.4 and 0.55, respectively. The input node features are normalized before they are fed into the GNNs and the corresponding explanation models. For the downstream task \textit{node classification}, only the labels of the nodes in the training set is available for all models during the training process. 
The trained GNN models with the best performance on the validation set are preserved for test and explanation.


\subsection{Detailed Experimental Settings}
\label{detail_settings}

\noindent \textbf{Implementation of GNN Models.} 
In our experiments, all GNN models are implemented in PyTorch~\cite{paszke2017automatic} with PyG (PyTorch Geometric)~\cite{pyglibrary}.
For the corresponding hyper-parameters, we set the value of weight decay as 5e-4, with the hidden dimension number and dropout rate being 64 and 0.6, respectively.
In addition, we set the learning rate and epoch number as 5e-2 and 200 for training.

\noindent \textbf{Implementation of ELEGANT.} 
ELEGANT is implemented in PyTorch~\cite{paszke2017automatic} with MIT license and all GNNs under ELEGANT are optimized through Adam optimizer~\cite{kingma2014adam} on Nvidia A6000. In our experiments, the sampling sizes of Gaussian noise and Bernoulli noise are 150 and 200, respectively.
All hyper-parameters for GNNs under ELEGANT are set as the same values as the hyper-parameters adopted for vanilla GNNs.
%
%
We set the confidence level as 0.7 for estimation, since a lower confidence level helps exhibit a clearer tendency of the change of certified budgets w.r.t. other parameters under a limited number of sampling size, considering the computational costs.
In the test phase, we set the sampled ratio for certification (from the nodes out of training and validation set) to be 0.9 to make the sampled size relatively large, in which way we include more nodes in the set of nodes to be certified.
In each run, we sample 100 times, and the value of FCR is averaged across three runs with different seeds.
In the parameter study for $\beta$, we reduce the total number of nodes used to ensure efficiency, while similar tendency is also observed in different settings.
Finally, considering the sizes of the three datasets, we set the nodes that are vulnerable (i.e., nodes whose perturbations are accessible to attackers) to be 5\% for German Credit and 1\% for others.

\subsection{Certification under Different Fairness Metrics} 

In Section 4.2, we present the experimental results based on the fairness metric of $\Delta_{\text{SP}}$, which measures the exhibited bias under the fairness notion of \textit{Statistical Parity}. We also perform the experiments based on $\Delta_{\text{EO}}$, which measures the exhibited bias under the fairness notion of \textit{Equal Opportunity}. We present the experimental results in~\cref{performance2}. We summarize the observations below.
(1) \textbf{Fairness Certification Rate (FCR).}
We observe that ELEGANT realizes large values of FCR (larger than 80\%) for all three GNN backbones and three attributed network datasets. Similar to our discussion in Section 4.2, this demonstrate that for nodes in any randomly sampled test set, we have a probability around or larger than 80\% to successfully certify the fairness level of the predictions yielded by the GNN model with our proposed framework ELEGANT. As a consequence, we argue that ELEGANT also achieves a satisfying fairness certification rate across all adopted GNN backbones and datasets on the basis of $\Delta_{\text{EO}}$.
In addition, we also observe that the German Credit dataset bears relatively larger values of FCR, while the values of FCR are relatively smaller with relatively larger standard deviation values on Recidivism and Credit Defaulter datasets.
A possible reason is that we set the threshold (i.e., $\eta$) as a value 25\% higher than the bias exhibited by the vanilla GNNs. Consequently, if the vanilla GNNs already exhibit a low level of bias, the threshold determined with such a strategy could be hard to satisfy under the added noise. This evidence indicates that the proposed framework ELEGANT tends to deliver better performance under scenarios where vanilla GNNs exhibit a high level of bias with the proposed strategy.
(2) \textbf{Utility.}
Compared with vanilla GNNs, certified GNNs with ELEGANT exhibit comparable and even higher node classification accuracy values in all cases. Therefore, we argue that the proposed framework ELEGANT does not significantly jeopardize the utility of the vanilla GNN models in certifying the fairness level of node classification.
(3) \textbf{Fairness.}
We observe that certified GNNs with ELEGANT are able to achieve better performances in terms of algorithmic fairness compared with those vanilla GNNs. This evidence indicates that the proposed framework ELEGANT also helps to mitigate the exhibited bias (by the backbone GNN models). We conjecture that such bias mitigation should be attributed to the same reason discussed in Section 4.2.


\begin{table*}[]
\small
  \vspace{-3mm}
\setlength{\tabcolsep}{8.02pt}
\renewcommand{\arraystretch}{1.0}
\centering
\vspace{-2mm}
\caption{Comparison between vanilla GNNs and certified GNNs under ELEGANT over three popular GNNs across three real-world datasets. Here ACC is node classification accuracy, and E- prefix marks out the GNNs under ELEGANT with certification. $\uparrow$ denotes the larger, the better; $\downarrow$ denotes the opposite. Different from the table in Section 4.2 (where the bias is measured with $\Delta_{\text{SP}}$), the bias is measured with $\Delta_{\text{EO}}$ here. Numerical values are in percentage, and the best ones are in bold.}
\label{performance2}
\begin{tabular}{cccccccccccc}
\toprule
            & \multicolumn{3}{c}{\textbf{German Credit}} &  & \multicolumn{3}{c}{\textbf{Recidivism}} &  & \multicolumn{3}{c}{\textbf{Credit Defaulter}} \\
            \cline{2-4}  \cline{6-8}  \cline{10-12}
            & \textbf{ACC ($\uparrow$)}   & \textbf{Bias ($\downarrow$)}   & \textbf{FCR ($\uparrow$)}   &  & \textbf{ACC ($\uparrow$)}   & \textbf{Bias ($\downarrow$)}   & \textbf{FCR ($\uparrow$)}   &  & \textbf{ACC ($\uparrow$)}   & \textbf{Bias ($\downarrow$)}   & \textbf{FCR ($\uparrow$)}   \\
            \midrule
\textbf{SAGE}         &67.3 $_{\pm 2.14}$       & 41.8 $_{\pm 11.0}$       & N/A    &  & 89.8 $_{\pm 0.66}$      & 6.09 $_{\pm 3.10}$       & N/A      &  & \textbf{75.9 $_{\pm 2.18}$}      &  10.4 $_{\pm 1.59}$       & N/A      \\
\textbf{E-SAGE} &\textbf{72.2 $_{\pm 1.26}$}  & \textbf{8.63 $_{\pm 6.15}$}       & 100 $_{\pm 0.00}$     &  & \textbf{90.8 $_{\pm 0.97}$}      &\textbf{3.12 $_{\pm 3.64}$}        & 81.0 $_{\pm 13.0}$      &  & 73.4 $_{\pm 0.61}$      & \textbf{7.18 $_{\pm 1.06}$}       & 88.7 $_{\pm 6.02}$      \\
\textbf{GCN}         & \textbf{59.6 $_{\pm 3.64}$}      & 35.0 $_{\pm 4.77}$       & N/A      &  & \textbf{90.5 $_{\pm 0.73}$}      & 6.35 $_{\pm 1.65}$       & N/A      &  & \textbf{65.8 $_{\pm 0.29}$}      & 13.5 $_{\pm 4.23}$       & N/A      \\
\textbf{E-GCN} & 58.8 $_{\pm 3.74}$      &\textbf{29.8 $_{\pm 6.82}$}        & 93.3 $_{\pm 8.73}$      &  & 89.3 $_{\pm 0.92}$      & \textbf{3.93 $_{\pm 3.12}$}       & 96.0 $_{\pm 4.97}$      &  & 63.5 $_{\pm 0.37}$      &\textbf{9.12 $_{\pm 0.95}$}        & 80.5 $_{\pm 14.5}$      \\
\textbf{JK}         & 63.3 $_{\pm 4.11}$      & 37.7 $_{\pm 15.9}$       & N/A      &  & \textbf{91.9 $_{\pm 0.54}$}      & 5.26 $_{\pm 3.25}$       & N/A      &  & 76.6 $_{\pm 0.69}$      & 8.04 $_{\pm 0.57}$       & N/A      \\
\textbf{E-JK} & \textbf{63.4 $_{\pm 3.68}$}      & \textbf{31.2 $_{\pm 15.5}$}       & 93.7 $_{\pm 8.96}$      &  &90.1 $_{\pm 0.55}$       & \textbf{2.54 $_{\pm 1.62}$}       & 83.7 $_{\pm 8.96}$      &  &\textbf{76.9 $_{\pm 0.86}$}       & \textbf{2.90 $_{\pm 2.04}$}       & 95.7 $_{\pm 4.80}$      \\
\bottomrule
\end{tabular}
\end{table*}

\subsection{Ordering the Inner and Outer Defense} 

\begin{figure}[t]
\centering
	\includegraphics[width=0.75\linewidth]{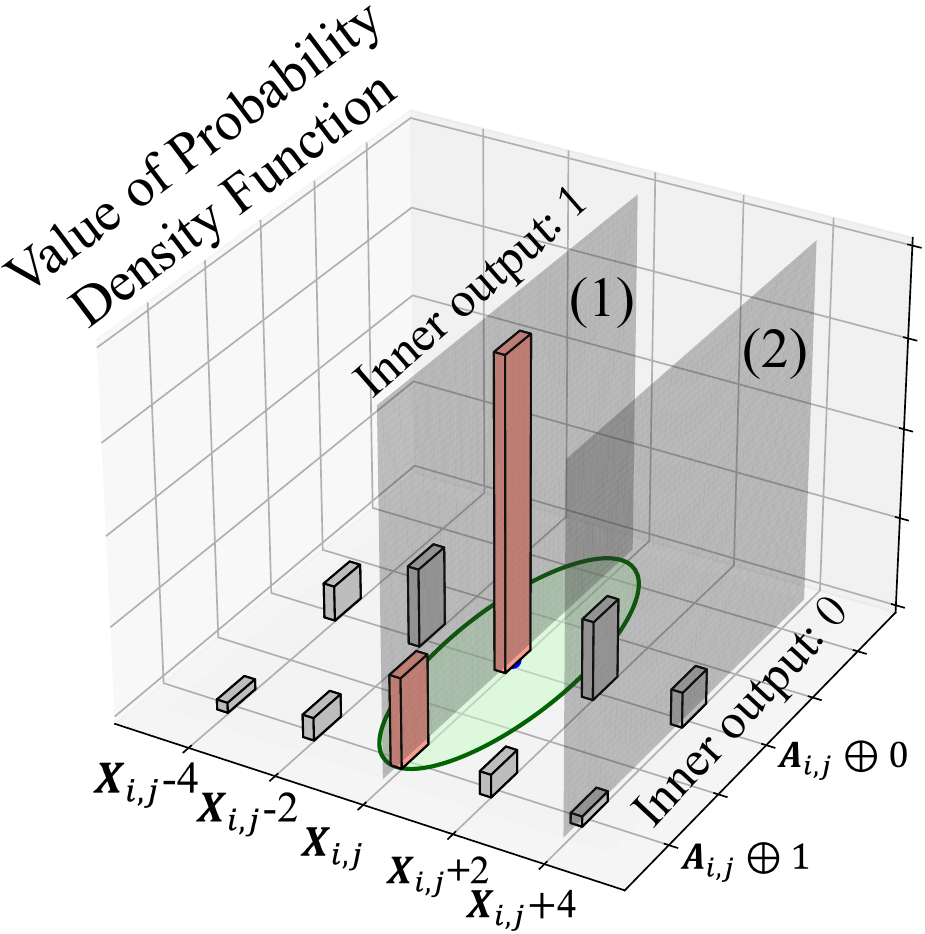}
	\caption{An example illustrating how ELEGANT works with a different order to achieve certified defense.}  
	\label{dual_certification_2}  
\end{figure}


We first review the general pipeline to achieve certified fairness defense. Specifically, we first model the fairness attack and defense by formulating the bias indicator function $g$. Then, we achieve certified defense over the node attributes for $g$, which leads to classifier $\tilde{g}_{\bm{X}}$. Finally, we realize certified defense for $\tilde{g}_{\bm{X}}$ over the graph topology, which leads to classifier $\tilde{g}_{\bm{A},\bm{X}}$. In general, we may consider the certified defense over node attributes and graph topology as the inner certified classifier and outer certified classifier, respectively.
Now, a natural question is: \textit{is it possible to achieve certified defense in a different order, i.e., first achieve certified defense over the graph topology (as the inner classifier), and then realize certified defense over the node attributes (as the outer classifier)?} Note that this is not the research focus of this paper, but we will provide insights about this question.
In fact, it is also feasible to achieve certified defense in the reversed order compared with the approach presented in our paper. We provide an illustration in~\cref{dual_certification_2}.
We follow a similar setting to plot this figure as in Section 3.3. Specifically, in case (1), both $\bm{A}_{i,j} \oplus 0$ and $\bm{A}_{i,j} \oplus 1$ lead to a positive outcome for $g$; in case (2), both $\bm{A}_{i,j} \oplus 0$ and $\bm{A}_{i,j} \oplus 1$ lead to a negative outcome. However, considering the Gaussian distribution around $\bm{X}_{i,j}$, samples will fall around case (1) with a much higher number compared with case (2). 
%
Hence, in this example, it would be reasonable to assume that the classifier with Bernoulli noise over graph topology (the inner certified classifier) will return 1 with a higher probability. This example illustrates how certification via a different order returns 1.

However, such a formulation bears higher computational costs in calculating the certified budgets. The reason is that we are able to utilize a closed-form solution to calculate $\epsilon_{\bm{X}}$ based on a set of Gaussian noise and the corresponding output from the bias indicator function.
However, based on a set of Bernoulli noise and the corresponding output from the bias indicator function, we will need to solve the optimization problem given in~\cref{a_r} to calculate $\epsilon_{\bm{A}}$, which bears a higher time complexity than calculating $\epsilon_{\bm{X}}$.
If we follow the strategy provided in Section 3.4 to calculate the inner and outer certification budgets, the certified budget of the inner certification will always be calculated multiple times, while the certified budget of the outer certification will only be calculated once.
Considering the high computational cost of calculating $\epsilon_{\bm{A}}$, we thus argue that it is more efficient to realize the certification over graph topology as the outer certified classifier.

\subsection{Time Complexity Analysis}

We now present a comprehensive analysis on the time complexity of ELEGANT. We present the analysis from both theoretical and experimental perspectives.

\noindent \textbf{Theoretical.} The time complexity is linear w.r.t. the total number of the random perturbations $N$, i.e., $\mathcal{O}(N)$. We perform 30,000 random perturbations over the span of node attributes and graph structure. We note that the actual running time is acceptable since the certification does not require re-training (which is the most costly process). In addition, all runnings do not rely on the prediction results from each other. Hence they can be paralleled altogether theoretically to further reduce the running time.

\noindent \textbf{Experimental.} We perform a study of running time, and we present the results in~\cref{running_time}. Specifically, we compare the running time of a successful certification under 30,000 random noise samples and a regular training-inference cycle with vanilla GCN. We observe that (1) although ELEGANT improves the computational cost compared with the vanilla GNN backbones, the running time remains acceptable; and (2) ELEGANT has less running time growth rate on larger datasets. For example, E-SAGE has around 10x running time on German Credit (a smaller dataset) while only around 4x on Credit Default (a larger dataset) compared to vanilla SAGE. Hence we argue that ELEGANT bears a high level of usability in terms of complexity and running time.

\subsection{Additional Results on Different GNN Backbones \& Baselines}

\label{explain_fairness}

We perform additional experiments over two popular GNNs, including APPNP~\cite{KlicperaBG19} and GCNII~\cite{chen2020simple}, to evaluate the generalization ability of ELEGANT onto different backbones. We present all numerical results in \cref{addi_acc} (in terms of accuracy), \cref{addi_fairness} (in terms of fairness), and \cref{addi_fcr} (in terms of FCR). We observe that ELEGANT achieves comparable utility, a superior level of fairness, and a large percentage of FCR. This verifies the satisfying usability of ELEGANT, which remains consistent with the paper.

In addition, we provide a detailed fairness comparison between ELEGANT and robust GNNs from \cite{jin2021node} and \cite{wu2019adversarial} in \cref{addi_baselines}. We observe that the best performances still come from the GNNs equipped with ELEGANT on all datasets. Hence we argue that ELEGANT exhibits satisfying performance in usability, which remains consistent with the discussion in the paper.

\noindent \textbf{Why ELEGANT Improves Fairness?} We note that improving fairness is a byproduct of ELEGANT, and our focus is to achieve certification over the fairness level of the prediction results. We now provide a detailed discussion about why fairness is improved here.  
First, existing works found that the distribution difference in the node attribute values and edge existence across different subgroups is a significant source of bias~\cite{dong2021edits,dai2021say,fan2021fair}. However, adding noise on both node attributes and graph topology may reduce such distributional divergence and mitigate bias. 
Second, As mentioned in \cref{certification_practice}, the proposed strategy to obtain the output predictions in ELEGANT is to select the fairest result among the output set $\hat{\mathcal{Y}}'$, where each output is derived based on a sample $\boldsymbol{\Gamma}_{\boldsymbol{A}}' \in \bar{\mathcal{A}}'$ (i.e., $\mathrm{argmin}_{\hat{\boldsymbol{Y}}'} \pi(\hat{\boldsymbol{Y}}', \mathcal{V}_{\text{tst}}) \;\; s.t. \; \hat{\boldsymbol{Y}}' \in \hat{\mathcal{Y}}'$). Such a strategy provides a large enough probability to achieve certification in light of Proposition 1. Meanwhile, we note that such a strategy also helps to significantly improve fairness since highly biased outputs are excluded.

\subsection{Complementary Results}

We provide the results in terms of $\Delta_{EO}$ for \cref{performance} in \cref{sup_eo}, and we present the results of the baselines for \cref{under_attacks} in \cref{sup_sp1}, \cref{sup_sp2}, \cref{sup_sp3}, and \cref{sup_sp4}.
For \cref{sup_eo}, we observe that ELEGANT does not constantly show a lower value of $\Delta_{EO}$. This is because the certification goal in \cref{performance} is $\Delta_{SP}$ instead of $\Delta_{EO}$. In addition, we note that debiasing existing GNN models is not the goal of this paper.
%
%
In addition, we provide the corresponding results in terms of accuracy for \cref{effectiveness-2} in \cref{acc_param_study_x} and \cref{acc_param_study_a}. We observe that although most performance remains stable, a stronger noise (i.e., larger $\sigma$ and smaller $\beta$) generally leads to worse but still comparable performance. This is consistent with the discussion in \cref{param_study}, and this has been taken into consideration in the discussion of the parameter selection strategy in \cref{param_study}.

\begin{table} [h]
\small
    \setlength{\tabcolsep}{8.5pt}
\renewcommand{\arraystretch}{1.0}
  \centering
  \caption{Comparison of running time (in seconds) on different datasets using different methods.}
  \label{running_time}
  \vspace{-3mm}
  \begin{tabular}{cccc}
\toprule
 & \textbf{German} & \textbf{Recidivism} & \textbf{Credit} \\
\midrule
    \textbf{SAGE} & 5.27 $\pm$ 0.38 & 34.14 $\pm$ 1.08 & 40.11 $\pm$ 0.36 \\
    \textbf{E-SAGE} & \textbf{53.23 $\pm$ 1.31} & \textbf{137.12 $\pm$ 58.66} & \textbf{157.51 $\pm$ 37.21} \\
    \textbf{GCN} & 5.59 $\pm$ 0.37 & 34.94 $\pm$ 1.16 & 40.59 $\pm$ 0.32 \\
    \textbf{E-GCN} & \textbf{53.79 $\pm$ 30.19} & \textbf{212.94 $\pm$ 10.38} & \textbf{214.11 $\pm$ 10.31} \\
    \textbf{JK} & 5.78 $\pm$ 0.43 & 34.68 $\pm$ 0.88 & 39.44 $\pm$ 1.56 \\
    \textbf{E-JK} & \textbf{59.99 $\pm$ 25.01} & \textbf{238.37 $\pm$ 1.81} & \textbf{252.99 $\pm$ 17.03} \\
\bottomrule
  \end{tabular}
\end{table}

\begin{table} [h]
\small
  \centering
    \setlength{\tabcolsep}{10.5pt}
\renewcommand{\arraystretch}{1.0}
  \caption{Performance comparison of classification accuracy. Numbers are in percentage.}
  \label{addi_acc}
  \vspace{-3mm}
  \begin{tabular}{cccc}
\toprule
    & \textbf{German} & \textbf{Recidivism} & \textbf{Credit} \\
\midrule
    \textbf{SAGE} & 67.3 $\pm$ 2.14    & 89.8 $\pm$ 0.66    & 75.9 $\pm$ 2.18    \\
    \textbf{E-SAGE} & \textbf{71.0 $\pm$ 1.27}    & \textbf{89.9 $\pm$ 0.90}    & \textbf{73.4 $\pm$ 0.50}    \\
    \textbf{GCN} & 59.6 $\pm$ 3.64    & 90.5 $\pm$ 0.73    & 65.8 $\pm$ 0.29    \\
    \textbf{E-GCN} & \textbf{58.2 $\pm$ 1.82}    & \textbf{89.6 $\pm$ 0.74}    & \textbf{65.2 $\pm$ 0.99}    \\
    \textbf{JK} & 63.3 $\pm$ 4.11    & 91.9 $\pm$ 0.54    & 76.6 $\pm$ 0.69    \\
    \textbf{E-JK} & \textbf{62.3 $\pm$ 4.07}    & \textbf{89.3 $\pm$ 0.33}    & \textbf{77.7 $\pm$ 0.27}    \\
    \textbf{APPNP} & 69.9 $\pm$ 2.17    & 95.3 $\pm$ 0.78    & 74.4 $\pm$ 3.05    \\
    \textbf{E-APPNP} & \textbf{69.4 $\pm$ 0.83}    & \textbf{95.9 $\pm$ 0.02} & \textbf{74.6 $\pm$ 0.32}    \\
    \textbf{GCNII} & 60.9 $\pm$ 1.00    & 90.4 $\pm$ 0.95    & 77.7 $\pm$ 0.22    \\
    \textbf{E-GCNII} & \textbf{60.4 $\pm$ 4.45}    & \textbf{88.8 $\pm$ 0.24}    & \textbf{77.6 $\pm$ 0.02}    \\
\bottomrule
  \end{tabular}
\end{table}

\begin{table}[h]
\small
  \centering
    \setlength{\tabcolsep}{10.5pt}
\renewcommand{\arraystretch}{1.0}
  \caption{Comparison of fairness (measured with $\Delta_{SP}$). Numbers are in percentage.}
  \label{addi_fairness}
  \vspace{-3mm}
  \begin{tabular}{cccc}
\toprule
    & \textbf{German} & \textbf{Recidivism} & \textbf{Credit} \\
\midrule
    \textbf{SAGE} & 50.6 $\pm$ 15.9    & 9.36 $\pm$ 3.15    & 13.0 $\pm$ 4.01    \\
    \textbf{E-SAGE} & \textbf{16.3 $\pm$ 10.9}    & \textbf{6.39 $\pm$ 2.85}    & \textbf{8.94 $\pm$ 0.99}    \\
    \textbf{GCN} & 37.4 $\pm$ 3.24    & 10.1 $\pm$ 3.01    & 11.1 $\pm$ 3.22    \\
    \textbf{E-GCN} & \textbf{3.52 $\pm$ 3.77}    & \textbf{9.56 $\pm$ 3.22}    & \textbf{7.28 $\pm$ 1.46}    \\
    \textbf{JK} & 41.2 $\pm$ 18.1    & 10.1 $\pm$ 3.15    & 9.24 $\pm$ 0.60    \\
    \textbf{E-JK} & \textbf{22.4 $\pm$ 1.95}    & \textbf{6.26 $\pm$ 2.78}    & \textbf{3.37 $\pm$ 2.64}    \\
    \textbf{APPNP} & 27.4 $\pm$ 4.81    & 9.71 $\pm$ 3.57    & 12.3 $\pm$ 3.14    \\
    \textbf{E-APPNP} & \textbf{13.1 $\pm$ 5.97}    & \textbf{2.23 $\pm$ 0.04} & \textbf{10.8 $\pm$ 0.07}    \\
    \textbf{GCNII} & 51.4 $\pm$ 0.36    & 9.70 $\pm$ 3.37    & 7.62 $\pm$ 0.29    \\
    \textbf{E-GCNII} & \textbf{24.9 $\pm$ 0.47}    & \textbf{3.78 $\pm$ 0.93}    & \textbf{1.72 $\pm$ 0.81}    \\
\bottomrule
  \end{tabular}
\end{table}

\begin{table}[h]
\small
  \centering
    \setlength{\tabcolsep}{12.5pt}
\renewcommand{\arraystretch}{1.0}
  \caption{Performance in FCR on different datasets and backbone GNNs. Numbers are in percentage.}
  \label{addi_fcr}
  \vspace{-3mm}
  \begin{tabular}{cccc}
\toprule
    & \textbf{German} & \textbf{Recidivism} & \textbf{Credit} \\
\midrule
    \textbf{E-SAGE} & 98.7 $\pm$ 1.89   & 94.3 $\pm$ 6.65   & 94.3 $\pm$ 3.3   \\
    \textbf{E-GCN} & 96.3 $\pm$ 1.89   & 96.0 $\pm$ 3.56   & 92.7 $\pm$ 5.19   \\
    \textbf{E-JK} & 97.0 $\pm$ 3.00   & 89.5 $\pm$ 10.5   & 99.3 $\pm$ 0.47   \\
    \textbf{E-APPNP} & 97.8 $\pm$ 3.14   & 87.1 $\pm$ 3.79 & 95.5 $\pm$ 6.43   \\
    \textbf{E-GCNII} & 94.7 $\pm$ 5.27   & 92.9 $\pm$ 9.93   & 99.0 $\pm$ 1.41   \\
\bottomrule
  \end{tabular}
\end{table}

\begin{table} [h]
\small
  \centering
    \setlength{\tabcolsep}{11.5pt}
\renewcommand{\arraystretch}{1.0}
  \caption{Comparison of fairness (measured with $\Delta_{SP}$). Numbers are in percentage.}
  \label{addi_baselines}
  \vspace{-2mm}
  \begin{tabular}{cccc}
\toprule
    & \textbf{German} & \textbf{Recidivism} & \textbf{Credit} \\
\midrule
    \textbf{SAGE} & 50.6 $\pm$ 15.9    & 9.36 $\pm$ 3.15    & 13.0 $\pm$ 4.01    \\
    \textbf{E-SAGE} & \textbf{16.3 $\pm$ 10.9}    & \textbf{6.39 $\pm$ 2.85}    & \textbf{8.94 $\pm$ 0.99}    \\
    \textbf{GCN} & 37.4 $\pm$ 3.24    & 10.1 $\pm$ 3.01    & 11.1 $\pm$ 3.22    \\
    \textbf{E-GCN} & \textbf{3.52 $\pm$ 3.77}    & \textbf{9.56 $\pm$ 3.22}    & \textbf{7.28 $\pm$ 1.46}    \\
    \textbf{JK} & 41.2 $\pm$ 18.1    & 10.1 $\pm$ 3.15    & 9.24 $\pm$ 0.60    \\
    \textbf{E-JK} & \textbf{22.4 $\pm$ 1.95}    & \textbf{6.26 $\pm$ 2.78}    & \textbf{3.37 $\pm$ 2.64}    \\
    \textbf{\cite{jin2021node}} & 14.8 $\pm$ 18.3 & 9.59 $\pm$ 0.65 & 3.84 $\pm$ 0.17 \\
    \textbf{\cite{wu2019adversarial}} & 3.66 $\pm$ 0.52 & 8.04 $\pm$ 2.97 & 7.10 $\pm$ 5.10 \\
\bottomrule
  \end{tabular}
\end{table}

\begin{table}[h]
\small
  \centering
    \setlength{\tabcolsep}{11.5pt}
\renewcommand{\arraystretch}{1.0}
  \caption{The $\Delta_{EO}$ of \cref{performance} in the paper. All numerical numbers are in percentage.}
  \label{sup_eo}
  \vspace{-2mm}
  \begin{tabular}{cccc}
\toprule
    & \textbf{German} & \textbf{Recidivism} & \textbf{Credit} \\
\midrule
    \textbf{SAGE} & 30.43 $\pm$ 0.07   & 3.71 $\pm$ 0.01   & 5.56 $\pm$ 0.03   \\
    \textbf{E-SAGE} & \textbf{12.21 $\pm$ 0.04}   & \textbf{6.95 $\pm$ 0.02}   & \textbf{7.18 $\pm$ 0.01}   \\
    \textbf{GCN} & 35.19 $\pm$ 0.07   & 5.06 $\pm$ 0.01   & 11.9 $\pm$ 0.02   \\
    \textbf{E-GCN} & \textbf{8.32 $\pm$ 0.03}   & \textbf{1.39 $\pm$ 0.01}   & \textbf{6.24 $\pm$ 0.02}   \\
    \textbf{JK} & 18.10 $\pm$ 0.13   & 3.02 $\pm$ 0.01   & 9.47 $\pm$ 0.02   \\
    \textbf{E-JK} & \textbf{23.68 $\pm$ 0.02}   & \textbf{2.74 $\pm$ 0.01}   & \textbf{2.55 $\pm$ 0.01}   \\
\bottomrule
  \end{tabular}
\end{table}


\begin{table}[]
\small
  \centering
    \setlength{\tabcolsep}{7.5pt}
\renewcommand{\arraystretch}{1.0}
  \caption{The results under ($2^0$, $10^{-1}$) in terms of node classification accuracy, AUC score, F1 score, $\Delta_{SP}$, and $\Delta_{EO}$ \cref{under_attacks}. All numerical numbers are in percentage.}
    \vspace{-2mm}
  \label{sup_sp1}
\begin{tabular}{cccccc}
\toprule
\textbf{($2^0$, $10^{-1}$)} & \textbf{Accuracy} & \textbf{AUC} & \textbf{F1 Score} & \textbf{$\Delta_{SP}$} & \textbf{$\Delta_{EO}$} \\
\midrule
\textbf{GCN}                & 58.4\%            & 66.4\%       & 63.9\%            & 41.4\%                 & 33.4\%                 \\
\textbf{NIFTY}              & 61.2\%            & 68.1\%       & 66.2\%            & 33.9\%                 & 13.3\%                 \\
\textbf{FairGNN}            & 55.2\%            & 62.2\%       & 61.4\%            & 16.4\%                 & 5.99\%  \\
\bottomrule
\end{tabular}
\end{table}

\begin{table}[]
\small
  \centering
    \setlength{\tabcolsep}{7.5pt}
\renewcommand{\arraystretch}{1.0}
  \caption{The results under ($2^1$, $10^{0}$) in terms of node classification accuracy, AUC score, F1 score, $\Delta_{SP}$, and $\Delta_{EO}$ \cref{under_attacks}. All numerical numbers are in percentage.}
    \vspace{-2mm}
  \label{sup_sp2}
\begin{tabular}{cccccc}
\toprule
\textbf{($2^1$, $10^{0}$)} & \textbf{Accuracy} & \textbf{AUC} & \textbf{F1 Score} & \textbf{$\Delta_{SP}$} & \textbf{$\Delta_{EO}$} \\
\midrule
\textbf{GCN}               & 58.4\%            & 66.4\%       & 63.9\%            & 41.4\%                 & 33.4\%                 \\
\textbf{NIFTY}             & 61.2\%            & 68.2\%       & 66.2\%            & 36.1\%                 & 13.3\%                 \\
\textbf{FairGNN}           & 55.2\%            & 62.2\%       & 61.4\%            & 16.8\%                 & 7.77\%  \\
\bottomrule
\end{tabular}
\end{table}

\begin{table}[]
\small
  \centering
    \setlength{\tabcolsep}{7.5pt}
\renewcommand{\arraystretch}{1.0}
  \caption{The results under ($2^2$, $10^{1}$) in terms of node classification accuracy, AUC score, F1 score, $\Delta_{SP}$, and $\Delta_{EO}$ for \cref{under_attacks}. All numerical numbers are in percentage.}
  \vspace{-2mm}
  \label{sup_sp3}
\begin{tabular}{cccccc}
\toprule
\textbf{($2^2$, $10^{1}$)} & \textbf{Accuracy} & \textbf{AUC} & \textbf{F1 Score} & \textbf{$\Delta_{SP}$} & \textbf{$\Delta_{EO}$} \\
\midrule
\textbf{GCN}               & 58.0\%            & 66.6\%       & 63.7\%            & 41.4\%                 & 37.8\%                 \\
\textbf{NIFTY}             & 61.2\%            & 68.1\%       & 66.0\%            & 42.1\%                 & 13.3\%                 \\
\textbf{FairGNN}           & 55.6\%            & 62.1\%       & 61.9\%            & 16.0\%                 & 9.56\%  \\
\bottomrule
\end{tabular}
\end{table}

\begin{table}[]
\small
  \centering
    \setlength{\tabcolsep}{7.5pt}
\renewcommand{\arraystretch}{1.0}
  \caption{The results under ($2^3$, $10^{2}$) in terms of node classification accuracy, AUC score, F1 score, $\Delta_{SP}$, and $\Delta_{EO}$ for \cref{under_attacks}. All numerical numbers are in percentage.}
    \vspace{-2mm}
  \label{sup_sp4}
\begin{tabular}{cccccc}
\toprule
\textbf{($2^3$, $10^{2}$)} & \textbf{Accuracy} & \textbf{AUC} & \textbf{F1 Score} & \textbf{$\Delta_{SP}$} & \textbf{$\Delta_{EO}$} \\
\midrule
\textbf{GCN}               & 58.0\%            & 67.7\%       & 63.7\%            & 42.6\%                 & 45.7\%                 \\
\textbf{NIFTY}             & 58.8\%            & 67.3\%       & 63.1\%            & 44.5\%                 & 19.4\%                 \\
\textbf{FairGNN}           & 54.4\%            & 61.4\%       & 61.0\%            & 16.9\%                 & 23.8\%  \\
\bottomrule
\end{tabular}
\end{table}

\begin{table}[h]
\small
  \centering
    \setlength{\tabcolsep}{6.8pt}
\renewcommand{\arraystretch}{1.0}
  \caption{Classification accuracy in \cref{param1} with different settings. Numbers are in percentage.}
  \label{acc_param_study_x}
  \vspace{-2mm}
  \begin{tabular}{ccccc}
\toprule
    & \textbf{5e-3} & \textbf{5e-2} & \textbf{5e-1} & \textbf{5e0} \\
\midrule
    \textbf{0} & 57.50 $\pm$ 1.50   & 57.51 $\pm$ 1.63   & 57.50 $\pm$ 1.58   & 55.67 $\pm$ 2.00   \\
    \textbf{1e-3} & 57.50 $\pm$ 1.51   & 57.51 $\pm$ 1.63   & 57.50 $\pm$ 1.58   & 55.67 $\pm$ 2.00   \\
    \textbf{5e-3} & 57.49 $\pm$ 1.52   & 57.50 $\pm$ 1.64   & 57.50 $\pm$ 1.58   & 55.67 $\pm$ 2.00   \\
    \textbf{1e-2} & 57.55 $\pm$ 1.50   & 57.51 $\pm$ 1.65   & 57.50 $\pm$ 1.58   & 55.67 $\pm$ 2.00   \\
    \textbf{5e-2} & N/A & 57.57 $\pm$ 1.59   & 57.50 $\pm$ 1.58   & 55.67 $\pm$ 2.00   \\
    \textbf{1e-1} & N/A & 57.53 $\pm$ 1.57   & 57.50 $\pm$ 1.59   & 55.67 $\pm$ 2.00   \\
    \textbf{5e-1} & N/A & N/A & 57.49 $\pm$ 1.60   & 55.67 $\pm$ 2.00   \\
    \textbf{1e0} & N/A & N/A & 57.40 $\pm$ 1.58   & 55.67 $\pm$ 2.00   \\
    \textbf{5e0} & N/A & N/A & N/A & 55.76 $\pm$ 1.86   \\
\bottomrule
  \end{tabular}
\end{table}

\begin{table}[h]
\small
  \centering
    \setlength{\tabcolsep}{6.5pt}
\renewcommand{\arraystretch}{1.0}
  \caption{Classification accuracy in \cref{param2} with different settings. Numbers are in percentage.}
  \label{acc_param_study_a}
  \vspace{-2mm}
  \begin{tabular}{ccccc}
\toprule
    & \textbf{0.6} & \textbf{0.7} & \textbf{0.8} & \textbf{0.9} \\
\midrule
    \textbf{0} & 63.71 $\pm$ 0.64   & 64.03 $\pm$ 0.66   & 65.87 $\pm$ 0.49   & 64.88 $\pm$ 0.46   \\
    \textbf{2}$^0$ & 63.71 $\pm$ 0.64   & 64.03 $\pm$ 0.66   & 65.87 $\pm$ 0.49   & 64.88 $\pm$ 0.46   \\
    \textbf{2}$^1$ & 63.67 $\pm$ 0.64   & 64.04 $\pm$ 0.67   & N/A & N/A \\
    \textbf{2}$^2$ & 63.69 $\pm$ 0.67   & N/A & N/A & N/A \\
    \textbf{2}$^3$ & N/A & N/A & N/A & N/A \\
    \textbf{2}$^4$ & N/A & N/A & N/A & N/A \\
\bottomrule
  \end{tabular}
\end{table}


\section{Additional Discussion}

\subsection{Why Certify A Classifier on top of An Optimized GNN? }

We note that the rationale of certified defense is to provably maintain the classification results against attacks. Under this context, most existing works on certifying an existing deep learning model focus on certifying a specific predicted label over a given data point. Here, the prediction results to be certified are classification results. Correspondingly, these works are able to certify the model itself. 

However, the strategy above is not feasible in our studied problem. This is because we seek to certify the level of fairness of a group of nodes. The value of such a group-level property cannot be directly considered as a classification result, and thus they are not feasible to be directly certified. Therefore, we proposed to first formulate a classifier on top of an optimized GNN. As such, achieving certification becomes feasible. In fact, this also serves as one of the contributions of our work.

\begin{table*}[h]
  \centering
      \vspace{-2mm}
    \setlength{\tabcolsep}{12.5pt}
  \caption{Experimental results on Pokec-z and Pokec-n datasets.}
    \vspace{-2mm}
  \label{addi_reuslts}
\begin{tabular}{cccccccc}
\toprule
\multicolumn{1}{l}{} & \multicolumn{3}{c}{\textbf{Pokec-z}}                                                       &   & \multicolumn{3}{c}{\textbf{Pokec-n}}                                                         \\
            \cline{2-4}  \cline{6-8}
\multicolumn{1}{l}{} & \textbf{ACC ($\uparrow$)}     & \textbf{Bias ($\downarrow$)} & \textbf{FCR ($\uparrow$)}  &    & \textbf{ACC ($\uparrow$)}     & \textbf{Bias ($\downarrow$)} & \textbf{FCR ($\uparrow$)}     \\
\midrule
\textbf{SAGE}        & 63.13 $\pm$ 0.37              & 6.29 $\pm$ 0.20 & -      &                        & 57.60 $\pm$ 2.74 & 6.43 $\pm$ 1.08 & -                             \\
\textbf{E-SAGE}      & 62.09 $\pm$ 2.22              & 4.18 $\pm$ 1.87 & 94.00 $\pm$ 5.66 &  & 60.74 $\pm$ 1.87 & 5.23 $\pm$ 0.13 & 91.50 $\pm$ 7.78 \\
\textbf{GCN}         & 64.89 $\pm$ 0.93 & 3.44 $\pm$ 0.16 & -        &                      & 59.86 $\pm$ 0.09 & 4.26 $\pm$ 0.40 & -                             \\
\textbf{E-GCN}       & 62.38 $\pm$ 0.26 & 1.52 $\pm$ 0.49 & 90.50 $\pm$ 0.71 &  & 59.83 $\pm$ 4.16 & 3.23 $\pm$ 1.20 & 94.00 $\pm$ 8.49 \\
\textbf{JK}          & 63.06 $\pm$ 1.00 & 7.89 $\pm$ 3.05 & -      &                        & 57.70 $\pm$ 1.05 & 8.81 $\pm$ 2.46 & -                             \\
\textbf{E-JK}        & 61.49 $\pm$ 2.55 & 3.63 $\pm$ 2.18 & 87.50 $\pm$ 2.12 &  & 61.19 $\pm$ 0.50 & 5.60 $\pm$ 0.01 & 93.00 $\pm$ 9.90 \\
\bottomrule
\end{tabular}
\end{table*}
\vskip -20ex

\subsection{What Is the Difference Between the Attacking Performance of GNNs and the Fairness of GNNs?}

In traditional attacks over the performance of GNNs, the objective of the attacker is simply formulated as having false predictions on as many nodes as possible, such that the overall performance is jeopardized. However, in attacks over the fairness of GNNs, whether the goal of the attacker can be achieved is jointly determined by the GNN predictions over all nodes. Such node-level dependency in achieving the attacking goal makes the defense over fairness attacks more difficult, since the defense cannot be directly performed at the node level but at the model level instead. Correspondingly, this necessitates (1) constructing an additional classifier as discussed in the previous reply, and (2) additional theoretical analysis over the constructed classifier as in Theorem 1 and 2 to achieve certification.

\subsection{Certification Without Considering the Binary Sensitive Attribute}

We utilize the most widely studied setting to assume the sensitive attributes are binary. However, our certification approach is not designed to be tailored to the sensitive attributes. Therefore, our approach can be easily extended to scenarios where the sensitive attributes are multi-class and continuous by adopting the corresponding fairness metric as the function $\pi(\cdot)$ in Definition 1.

\subsection{How Do the Main Theoretical Findings Differ From Existing Works on Robustness Certification of GNNs on Regular Attacks?}

Most existing works for robustness certification can only defend against attacks on either node attributes or graph structure. Due to the multi-modal input data of GNNs, existing works usually fail to handle the attacks over node attributes and graph structure at the same time. However, ELEGANT is able to defend against attacks over both data modalities. This necessitates using both continuous and discrete noises for smoothing and the analysis for joint certification in the span of the two input data modalities.



\subsection{Difference with Existing Works}

Here we mainly focus on discussing the difference between this work and \cite{bojchevski2020efficient}.
We note that (1) the randomized smoothing technique adopted in \cite{bojchevski2020efficient} is different from the proposed randomized smoothing approach on the graph topology in this paper and (2) the techniques in \cite{bojchevski2020efficient} tackle a different problem from this paper. We elaborate on more details below. %
The techniques in \cite{bojchevski2020efficient} are different from this paper. Although both randomized smoothing approaches are able to handle binary data, we note that the randomized smoothing approach proposed in \cite{bojchevski2020efficient} is data-dependent. However, the proposed randomized smoothing approach in this paper is data-independent. We note that in practice, a data-independent approach enables practitioners to pre-generate noises, which enhances usability.

The studied problem in \cite{bojchevski2020efficient} is different from this paper. Although the authors claimed to achieve a joint certificate for graph topology and node attributes in \cite{bojchevski2020efficient}, all node attributes are assumed to be binary, which can only be applied to cases where these attributes are constructed as bag-of-words representations (as mentioned in the second last paragraph in the Introduction of \cite{bojchevski2020efficient}). However, in this work, we follow a more realistic setting where only graph topology is assumed to be binary while node attributes are considered as continuous. This makes the problem more difficult to handle, since different strategies should be adopted for different data modalities. In summary, compared with \cite{bojchevski2020efficient}, the problem studied in this paper is more realistic and more suitable for GNNs.

Based on the discussion above, we would like to note that no existing work can be directly adopted to tackle the studied problem in this paper, which further evidenced the novelty of this work.

\subsection{Additional Experiments on Other Datasets}

To further validate the performance of the proposed method, we also perform experiments with the same commonly used popular GNN backbone models (as Section 4.2) on two Pokec datasets, namely Pokec-z and Pokec-n. We present the experimental results in \cref{addi_reuslts}, where all numerical numbers are in percentage. We observe that (1) the GNNs equipped with ELEGANT achieve comparable node classification accuracy; (2) the GNNs equipped with ELEGANT achieve consistently lower levels of bias; and (3) the values of the Fairness Certification Rate (FCR) for all GNNs equipped with ELEGANT exceed 90\%, exhibiting satisfying usability. All three observations are consistent with the experimental results and the corresponding discussion presented in Section 4.2. Therefore, we argue that the effectiveness of ELEGANT is not determined by the dataset and is well generalizable over different datasets.

\subsection{Scalability of ELEGANT}

In this subsection, we discuss the scalability of ELEGANT.
Specifically, we note that if the Gaussian and Bernoulli noise is directly added over the whole graph, scaling to larger graphs would be difficult. However, the proposed approach can be easily extended to the batch training case, which has been widely adopted by existing scalable GNNs. Specifically, a commonly adopted batch training strategy of scalable GNNs is to only input a node and its surrounding subgraph into the GNN, since the prediction of GNNs only depends on the information of the node itself and its multi-hop neighbors, and the number of hops is determined by the layer number of GNNs. Since the approach proposed in our paper aligns with the basic pipeline of GNNs, the perturbation can also be performed for each specific batch of nodes. In this case, all theoretical analyses in this paper still hold, since they also do not rely on the assumption of non-batch training. Therefore, we would like to argue that the proposed approach can be easily scaled to large graphs.

\end{document}